%% file: neurips_2026.tex
\documentclass{article}

\PassOptionsToPackage{numbers, compress}{natbib}
\usepackage[preprint]{neurips_2026}

\usepackage[utf8]{inputenc} 
\usepackage[T1]{fontenc}    
\usepackage{hyperref}       
\usepackage{url}            
\usepackage{booktabs}       
\usepackage{amsfonts}       
\usepackage{nicefrac}       
\usepackage{microtype}      
\usepackage{xcolor}         

\usepackage{graphicx}
\usepackage{amsmath}

\usepackage{amsthm}

\newtheorem{theorem}{Theorem}[section]
\newtheorem{lemma}[theorem]{Lemma}
\newtheorem{proposition}[theorem]{Proposition}
\newtheorem{corollary}[theorem]{Corollary}

\theoremstyle{remark}
\newtheorem{remark}[theorem]{Remark}

\usepackage{subcaption}

\title{S2T-RLHF: Hierarchical Credit Assignment for Stable Preference-Based RLHF}

\author{%
  Wei Chen$^{1}$ \quad
  Guanghui Zhu$^{1}$\thanks{Corresponding author.} \quad
  Yafei Li$^{2}$ \quad
  Limin Wang$^{1}$ \quad
  Yihua Huang$^{1}$ \\
  $^{1}$State Key Laboratory for Novel Software Technology, Nanjing University \\
  $^{2}$School of Computer Science and Artificial Intelligence, Zhengzhou University \\
  \texttt{ischenwei@outlook.com, zgh@nju.edu.cn, ieyfli@zzu.edu.cn,}  \\\texttt{lmwang@nju.edu.cn, yhuang@nju.edu.cn}
}

\begin{document}

\maketitle

\begin{abstract}
      
      Reinforcement learning from human feedback (RLHF) with preference-based reward models often exhibits unstable training dynamics. A key contributing factor is that standard RLHF relies on a single sequence-level scalar reward, which is propagated to token-level policy updates and leaves credit assignment within a response inherently ambiguous. Recent work has attempted to address this issue by refining rewards into denser token-level supervision, often relying on the implicit assumption that finer-grained credit assignment improves optimization. We argue that this assumption is incomplete: when preference signals are noisy and only defined at the response level, overly fine-grained reward refinement can amplify reward uncertainty and destabilize learning. To address this problem, we propose a granularity-aware principle for hierarchical credit assignment, emphasizing stability-oriented reward design rather than maximal allocation precision. Under this principle, sentences serve as a natural intermediate granularity, balancing semantic coherence with robustness to token-level noise. Guided by this view, we introduce S2T-RLHF. This sentence-to-token reward decomposition framework first allocates sequence-level preference rewards across sentences and then applies bounded token-level refinement within each sentence, without reward-model retraining or token-level supervision. Experiments across multiple datasets and optimization settings show that S2T-RLHF improves training stability and robustness while maintaining competitive preference alignment.
\end{abstract}

\input{chapters/introduction}

\input{chapters/related}

\input{chapters/methdology}

\input{chapters/token}

\input{chapters/experiment}

\input{chapters/discussion}
\input{chapters/conclusion}

\bibliographystyle{unsrtnat}

\bibliography{reference}

\newpage
\appendix

\input{appendix/app_preliminarie}

\input{appendix/app_proof}

\input{appendix/app_method}

\input{appendix/app_analysis}

\input{appendix/app_end}

\input{appendix/app_details}

\input{appendix/app_metrics_stable}

\input{appendix/app_stable}

\input{appendix/app_ex_tab}

\input{appendix/app_case}

\input{appendix/app_compute}



\end{document}

%% file: chapters/introduction.tex
\section{Introduction}
\label{Introduction}

Reinforcement learning from human feedback (RLHF) has become a central paradigm for preference-based alignment of large language models (LLMs), enabling substantial improvements in helpfulness, harmlessness, and human preference alignment~\cite{christiano2017deep, ziegler2019fine, ouyang2022training, sharma2025rlhf}. Despite this empirical success, RLHF often exhibits pronounced instability in long-horizon language generation. Common pathologies include length bias, reward concentration, and occasional reward collapse~\cite{chai2024ma, moskovitz2024confronting, yang2024regularizing, miao2024inform}. These issues persist despite careful hyperparameter tuning and regularization, suggesting underlying causes beyond implementation artifacts.

A key limitation of the standard RLHF pipeline lies in its reward formulation. The learned reward model assigns a single scalar score to an entire generated response, providing supervision only at the sequence level. While effective for optimizing global preferences, this scalar feedback offers no guidance on how credit should be assigned within a long generation trajectory~\cite{zhong2025dpo}. Consequently, reinforcement learning must propagate a global preference signal across many decoding steps, often resulting in high-variance gradients, unstable policy updates, and limited controllability over fine-grained generation behavior.

We observe that many of these failures arise from a structural mismatch between the semantic granularity of language and the granularity at which rewards are assigned and optimized. At its core, this mismatch reflects a tension between how human preferences are expressed and how language models generate text. Human raters typically judge responses through sentences or semantic units, such as whether a claim is correct, whether an instruction is followed, whether a refusal is appropriate, or whether the tone is safe and helpful. In contrast, LLMs generate discrete tokens whose isolated contributions are sub-symbolic and context-dependent, and often lack interpretable meaning outside their surrounding linguistic structure. Directly deriving token-level rewards from a sequence-level scalar is therefore ill-posed. The global score entangles sentence-level semantic quality with token-level realizations, making fine-grained credit assignment inherently ambiguous and unstable.

Recent efforts~\cite{mao2025information, xie2025capo} have largely treated RLHF instability as a problem of insufficiently dense supervision, motivating increasingly detailed forms of token-level credit assignment. This line of work is often driven by the implicit assumption that finer-grained credit assignment is inherently beneficial for optimization. However, this assumption conflates precision with stability. When preference signals are noisy and defined only at the response level, refining rewards beyond the semantic resolution at which preferences are expressed can amplify reward uncertainty, propagate local artifacts from the reward model, and destabilize learning dynamics rather than improve them.

We therefore suggest that stabilizing RLHF does not hinge on pursuing maximal credit-assignment resolution, but rather on designing reward signals whose granularity aligns with the semantic structure of language. From this perspective, reward decomposition should prioritize semantically meaningful units that support stable optimization, rather than maximizing allocation precision. This granularity-aware perspective reframes dense reward design in RLHF as a problem of semantic alignment and learning stability, rather than one of credit-assignment precision alone.

To operationalize this perspective, we adopt a hierarchical view of credit assignment that respects the semantic structure of language. Rather than treating dense reward design as a flat allocation problem over tokens, we emphasize that credit assignment should reflect the inherent hierarchy of natural language: sentences express coherent semantic intent, while tokens refine that intent through lexical and syntactic choices.

Based on this principle, we present Sentence-to-Token RLHF (S2T-RLHF), a concrete instantiation of hierarchical credit assignment for stable preference-based RLHF. S2T-RLHF first allocates a sequence-level preference reward over semantically stable sentence-level units through a structured cooperative allocation mechanism, and then applies bounded token-level refinement within each sentence. The method uses only sequence-level preference rewards and requires no reward-model retraining or token-level supervision. Importantly, S2T-RLHF is intended not as the only possible realization of this principle, but as a simple and practical instantiation of stability-oriented hierarchical credit assignment. We further provide theoretical discussion in Appendices~\ref{app:discussion}--\ref{app:end_to_end}, analyzing how sequence-level rewards can induce global advantage alignment, long-horizon variance amplification, and length-dependent gradient scaling, and how the two-stage S2T design mitigates these effects.
We summarize our main contributions as follows:

$\bullet$ We propose a credit assignment design principle for RLHF, emphasizing that training stability should be achieved by aligning reward granularity with the semantic structure of language, rather than by pursuing increasingly fine-grained credit assignment.
\vskip -0.05in

$\bullet$ Guided by this principle, we develop S2T-RLHF, a hierarchical semantic credit assignment method that decomposes a sequence-level preference signal into sentence-level allocations followed by token-level refinement, enabling dense and semantically grounded supervision without requiring additional annotations or reward-model retraining.
\vskip -0.05in

$\bullet$ We empirically evaluate S2T-RLHF across multiple datasets and optimization settings, showing that hierarchical credit assignment improves training stability and robustness while maintaining competitive preference alignment against existing credit assignment methods.

The project page is available at \href{https://pasalab.github.io/S2T-RLHF}{\texttt{https://pasalab.github.io/S2T-RLHF}}

%% file: chapters/related.tex
\vspace{-0.2cm}
\section{Related Work}
\label{related work}
\vspace{-0.2cm}

\paragraph{Large Language Models and Preference Alignment.}
Large language models pretrained on large-scale corpora have achieved strong performance in instruction following and open-ended generation~\cite{radford2019language, brown2020language, wei2021finetuned}. However, maximum-likelihood training, while effective for next-token prediction, cannot fully capture desired behavioral properties~\cite{christiano2017deep}, including subjective quality judgments~\cite{stiennon2020learning, kim2024spread}, normative constraints~\cite{ngo2022alignment, wachi2024stepwise}, and task-specific preferences~\cite{zhong2024panacea, kobalczyk2025preference}. This mismatch has motivated preference alignment, which aims to guide model behavior toward outputs that better reflect such criteria. A prominent framework is reinforcement learning from human feedback (RLHF), which learns a reward model from preference comparisons and optimizes the policy via reinforcement learning~\cite{christiano2017deep, ouyang2022training}. Related paradigms include direct policy optimization~\cite{rafailov2023direct}, contrastive preference learning~\cite{hejna2023contrastive}, odds ratio preference optimization~\cite{hong2024orpo}, and other feedback-driven methods~\cite{shao2024deepseekmath}.
\vspace{-0.3cm}
\paragraph{RLHF and Fine-Grained Reward Allocation.}
RLHF is a widely studied framework for aligning language models with preference signals from human feedback. In a typical pipeline, a reward model is trained from preference comparisons~\cite{gao2023scaling, bahdanau2018learning}, and the language model is then optimized with policy-gradient methods~\cite{ouyang2022training, rafailov2023direct, yao2023retroformer}, most commonly proximal policy optimization (PPO)~\cite{schulman2017proximal}, to maximize the predicted reward. This framework is widely used as a standard approach for preference-based policy optimization in language generation~\cite{ouyang2022training, wu2023privately}.

However, conventional RLHF assigns a single scalar reward to each generated response, producing a coarse and sparse learning signal that complicates credit assignment and can destabilize optimization. Recent work has explored fine-grained reward mechanisms to address this limitation. \citeauthor{chan2024dense} leverage reward-model attention weights to densify sequence-level rewards at the token level~\cite{chan2024dense}, but attention weights may not reliably reflect causal token contributions. \citeauthor{yoon2024tlcr} propose TLCR~\cite{yoon2024tlcr}, which trains a discriminator to assign token-level rewards, requiring an additional auxiliary model. \citeauthor{zhong2025dpo} reformulate RLHF as a token-wise MDP and propose RTO~\cite{zhong2025dpo}, which derives token-level rewards from preference data but inherits limitations from DPO's preference modeling and introduces complexity in the multi-stage pipeline. \citeauthor{chai2024ma} propose MA-RLHF~\cite{chai2024ma}, which groups tokens into macro actions to alleviate long-horizon credit assignment, but predefined macro actions remain coarse and may limit semantic adaptability. \citeauthor{cao2025scar} apply Shapley values to distribute sequence-level rewards based on marginal token contributions~\cite{cao2025scar}; however, the resulting credits are tightly coupled to reward-model scalar evaluations and can be sensitive to reward noise. Beyond these representative approaches, a range of recent methods have investigated fine-grained reward allocation~\cite{li2024r3hf, rashid2024critical, mao2025information, xie2025capo}. Despite their different formulations, most existing methods primarily pursue denser token-level supervision, leaving underexplored the question of what reward granularity is semantically meaningful and optimization-stable for preference-based RLHF.

%% file: chapters/methdology.tex
\vspace{-0.25cm}
\section{Methodology}
\label{Methodology}
\vspace{-0.2cm}
\subsection{Problem Setup}
\vspace{-0.2cm}
Given a prompt $x$, a policy $\pi_\theta$ generates a response $y=(t_1,\ldots,t_T)$ according to $\pi_\theta(\cdot|x)$. A preference-based reward model $R_\phi(x,y)$ assigns a scalar sequence-level reward to the full response. Standard RLHF optimizes the policy with a KL-regularized objective,
\[
\max_\theta\;
\mathbb{E}_{x\sim\mathcal{D},\,y\sim\pi_\theta(\cdot|x)}
\left[
R_\phi(x,y)
-\beta\,\mathrm{KL}\!\left(
\pi_\theta(\cdot|x)\,\|\,\pi_{\rm ref}(\cdot|x)
\right)
\right].
\]
S2T-RLHF keeps the reward model fixed and decomposes this sequence-level signal into sentence- and token-level credit signals for policy optimization. Background on Nash bargaining and the Dirichlet distribution is provided in Appendix~\ref{app:preliminaries}.
\begin{figure*}[t]
  \centering
  \vspace{-0.2cm}
  \includegraphics[width=0.95\textwidth]{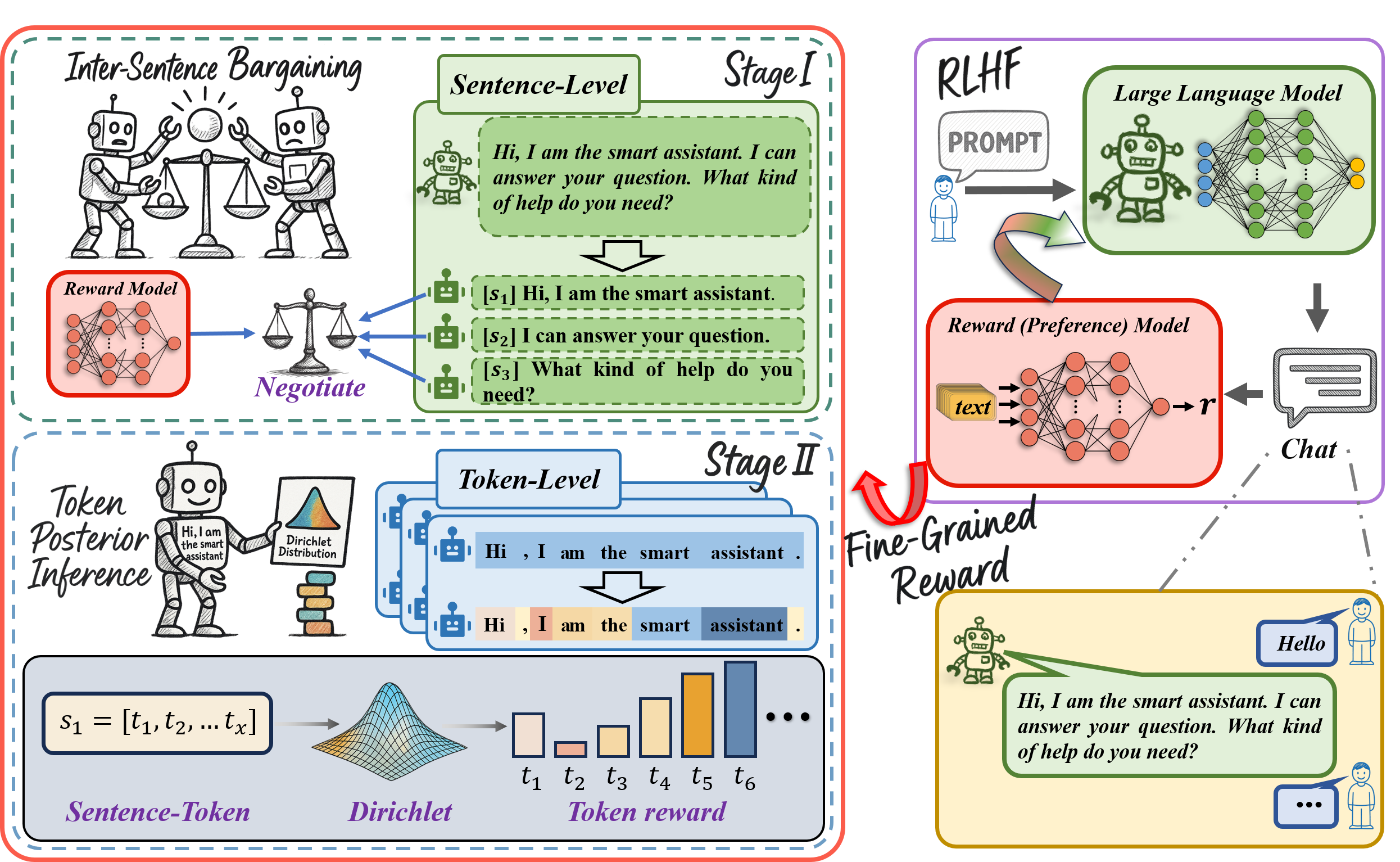}
  \caption{Overview of S2T-RLHF. A global sequence-level reward is first decomposed into sentence-level credits via inter-sentence bargaining (Stage I), and then refined into token-level rewards through Dirichlet posterior inference (Stage II), enabling stable and interpretable credit assignment in RLHF without modifying the reward model.}
  \label{pic:framework}
  \vspace{-0.5cm}
\end{figure*}

\subsection{Semantically-Grounded Reward Decomposition}

In standard RLHF, a single sequence-level scalar reward provides only coarse supervision over a long generation trajectory, leaving fine-grained credit assignment fundamentally underspecified. This mismatch often leads to unstable optimization dynamics, overly localized updates, and limited interpretability of learned behaviors. To address this challenge, we propose \textbf{S2T-RLHF}, a two-stage hierarchical reward decomposition framework that transforms a sequence-level preference signal into semantically grounded dense rewards without requiring additional annotations, reward-model retraining, or partial-sequence queries.

As illustrated in Fig.~\ref{pic:framework}, S2T-RLHF decomposes reward assignment across semantic granularities. In the first stage, the global sequence-level reward is allocated across sentences, producing a structured and interpretable intermediate representation that captures the relative semantic contributions of reasoning steps or conversational units while preserving total reward mass. In the second stage, conditioned on this sentence-level allocation, each sentence reward is redistributed across its constituent tokens, enabling fine-grained credit assignment aligned with local lexical and syntactic realizations. By explicitly separating semantic allocation from token-level refinement, S2T-RLHF provides stable, semantically coherent dense rewards suitable for RLHF optimization.

\subsection{Semantic Granularity and Design Rationale}
\vspace{-0.2cm}
Sentence-level and token-level representations capture different aspects of response quality. At the sentence level, complete semantic units and logical coherence shape the global structure and overall quality of a response. At the token level, variations mainly affect local precision and the expression of critical information. Thus, reward assignment at a single granularity cannot capture both global semantic quality and fine-grained lexical accuracy.

At the sentence level, reward decomposition should emphasize fairness across sentences. Notably, this goal runs counter to the standard reinforcement learning intuition of concentrating reward on the most salient components. Rather than amplifying a single high-reward sentence, sentence-level allocation aims to prevent semantically important sentences from being under-rewarded or suppressed. Relying solely on a sequence-level scalar reward can introduce length bias, as longer sentences exert disproportionate influence by containing more tokens. Moreover, sentences often play distinct but complementary roles within a logical structure, and uneven reward allocation can suppress essential components—an effect exacerbated by coupled instabilities in sequence-level RLHF, including advantage alignment, variance amplification, and length-dependent scaling (Appendix~\ref{app:discussion}). These considerations motivate a sentence-level allocation principle that balances contributions across sentences under a shared reward budget.

In contrast, token-level credit assignment should emphasize relative importance rather than fairness. Tokens differ substantially in information density: key nouns, verbs, entities, numerical values, and logical operators carry significantly higher semantic weight than function words or fillers. Errors in critical tokens (e.g., factual entities or negations) can drastically alter meaning, whereas minor lexical variations often have limited impact. Uniform token-level reward assignment, therefore, suffers from signal dilution, obscuring informative learning signals. Token-level rewards should instead reflect relative importance under uncertainty, enabling sparse yet informative credit assignment within each sentence. This asymmetry between sentence-level allocation and token-level allocation motivates the use of distinct mechanisms in the two stages of S2T-RLHF.

\vspace{-0.15cm}
\subsection{Sentence-Level Reward Allocation}
\vspace{-0.15cm}
RLHF assigns rewards at the sequence level, leaving the contribution of individual sentences underdetermined. We resolve this ambiguity by formulating sentence-level reward allocation as a cooperative bargaining problem, where sentences negotiate their shares of a fixed global reward based on their semantic influence. We use Nash bargaining not for its game-theoretic interpretation, but because it induces a scale-invariant, symmetric, and smooth allocation under a fixed reward budget. Other mechanisms with similar properties could be substituted.

\paragraph{Sentence embedding and semantic perturbation.}
Given a response $x=[s_1,\dots,s_K]$, each sentence $s_i$ is mapped to a contextualized semantic embedding $z_i=\mathrm{Enc}(s_i)\in\mathbb{R}^d$ using the LLM encoder. These embeddings capture sentence-level semantics and form the atomic units for credit assignment. To estimate how the sequence-level reward depends on individual sentences, we construct semantic perturbations by replacing $s_i$ with a set of alternative variants $\{s_i^{(k)}\}$ (Appendix~\ref{app:semantic_perturbation}). For each perturbation, we measure the induced changes in both reward and representation:
\begin{equation}
    \Delta R_{i,k} = R(x) - R(x^{(i\rightarrow k)}), \qquad
    \delta_{i,k} = z_i - z_i^{(k)}.
\end{equation}
These paired differences capture the local sensitivity of the global reward to semantic variations of sentence $i$ under the model's contextual encoding.

\paragraph{Latent semantic influence estimation.}
The perturbation signals above provide local observations of how sentence-level semantic changes affect the sequence-level reward. We summarize these observations by estimating a latent \emph{semantic influence vector} $v_i\in\mathbb{R}^d$ for each sentence via ridge-regularized regression:
\begin{equation}
    v_i = \arg\min_{v}\sum_{k} \big(\Delta R_{i,k} - \langle v,\delta_{i,k}\rangle\big)^2 + \lambda\|v\|^2.
\end{equation}
This vector captures a reward-sensitive semantic direction for sentence $i$, reflecting how the reward model responds to small semantic perturbations and to the LLM's autoregressive contextual structure. The semantic bargaining framework does not depend on any specific analytic properties of the reward model. To model inter-sentence dependencies in the reward-relevant semantic space, we represent each sentence by its influence vector and define
\begin{equation}
\label{eq:interaction_matrix}
    M = V^\top V, \qquad V=[v_1,\dots,v_K]\in\mathbb{R}^{d\times K},
\end{equation}
where $M_{ij}$ measures the degree of alignment between the influence vectors of sentences $i$ and $j$.

\paragraph{Semantic bargaining formulation.}
We introduce positive bargaining coefficients $\beta\in\mathbb{R}^K_{>0}$ to form a cooperative semantic direction from sentence-level influences, and define the latent utility of sentence $i$ as its alignment with this direction:
\begin{equation}
    h = \sum_{j=1}^K \beta_j v_j, \qquad
    \tilde u_i = \langle v_i, h\rangle = (M\beta)_i, \qquad
    u_i = \max(\tilde u_i,\varepsilon),
\end{equation}
where $\varepsilon>0$ ensures well-defined logarithms and reciprocal updates. To allocate reward proportionally to these latent utilities, we maximize the multiplicative bargaining objective
\begin{equation}
    \mathcal{B}(\beta)=\sum_{i=1}^K \log u_i.
\end{equation}

This objective follows the spirit of classical Nash bargaining and is used here as a principled allocation criterion rather than a global optimization guarantee. It yields a scale-invariant and symmetric cooperative allocation in the latent utility space. In practice, we compute the bargaining coefficients through a reciprocal fixed-point condition. In the idealized case without the $\epsilon$-safeguard, any positive fixed point satisfies
\begin{equation}
    M\beta = \beta^{\odot -1},
    \qquad \text{equivalently} \qquad
    u_i\beta_i = 1 \quad \forall i,
\end{equation}
where $u=M\beta$. This Nash-style balance condition does not require all latent utilities to be equal, and therefore allows non-uniform sentence-level allocation. Proposition~\ref{prop:kkt} gives the formal stationarity characterization.

\paragraph{Fixed-point solver and reward assignment.}
Instead of solving for $\tau$ explicitly, we approximate the bargaining equilibrium via a smooth damped fixed-point iteration:
\begin{equation}
\beta^{(t+1)} = (1-\eta)\,\beta^{(t)} + \eta\, \big(\max(M\beta^{(t)},\varepsilon)\big)^{\odot -1},
\label{eq:bargain_fp}
\end{equation}
where $\eta\in(0,1]$ and $\varepsilon>0$. The update is used as a numerical solver rather than a new optimization algorithmic contribution. The resulting fixed-point equation admits at least one solution in the positive orthant under mild regularity conditions (Corollary~\ref{cor:existence}), and the safeguarded iteration is well-defined and preserves positivity (Lemma~\ref{lem:positivity}); detailed diagnostics are provided in Appendix~\ref{app:positivity}. After convergence, sentence-level rewards are obtained by normalizing the latent utilities:
\begin{equation}
r_i = \frac{u_i}{\sum_j u_j}\,R(x), \qquad u = \max(M\beta,\varepsilon).
\end{equation}
These normalized values form a semantically grounded decomposition of the global reward, preserve the total reward mass $\sum_i r_i=R(x)$, and serve as the input for the token-level inference stage.

%% file: chapters/token.tex
\subsection{Token-Level Reward Allocation}

Although sentence-level rewards provide a useful intermediate signal, they do not specify how credit should be assigned within a sentence. Since human preferences may depend on a small subset of decisive tokens, uniform allocation can dilute informative learning signals. We therefore formulate token-level refinement as a simplex-constrained allocation problem. Given contextualized token representations, a lightweight Dirichlet Token Allocation Network (DTAN) predicts non-negative token weights that sum to one within each sentence. The resulting allocation is deterministic, normalized, differentiable, and serves as a bounded refinement of the sentence-level reward.

\paragraph{Dirichlet prior from contextual representations.}
Let $h_{s,t}$ denote the hidden representation of token $t$ in sentence $s$. DTAN maps $h_{s,t}$ to a positive Dirichlet concentration parameter:
\begin{equation}
    \alpha_{s,t} = \mathrm{softplus}(f_\phi(h_{s,t})),
\end{equation}
where $f_\phi$ is a lightweight multilayer perceptron. The concentration parameters encode a contextual prior over token contributions: larger values indicate stronger allocation weight, while smaller values reflect lower confidence. Grounding this prior in token representations allows DTAN to capture semantic and positional cues within each sentence.

\paragraph{Dirichlet-based token allocation.}
Given the concentration parameters, we model token allocation within sentence $s$ as $c_s \sim \mathrm{Dir}(\alpha_{s,1},\ldots,\alpha_{s,L_s})$, where $L_s$ is the sentence length. We use the Dirichlet expectation as the deterministic token contribution:
\[
c_{s,t}=\frac{\alpha_{s,t}}{\sum_{k=1}^{L_s}\alpha_{s,k}}.
\]
Thus $c_{s,t}\ge 0$ and $\sum_{t=1}^{L_s}c_{s,t}=1$, yielding a stable and differentiable allocation of sentence-level reward mass across tokens.

\paragraph{Reward decomposition into token-level signals.}
In practice, the decomposition is implemented as per-token modulation of the sequence-level advantage in PPO, with additional clipping for stability (Appendix~\ref{app:advantage}). Conceptually, if the sentence-level reward is $R^{(s)}$, the token-level signal assigned to token $t$ is
\begin{equation}
    r_{s,t} = R^{(s)} \cdot c_{s,t}.
\end{equation}
This decomposition preserves reward mass, $\sum_{t=1}^{L_s} r_{s,t}=R^{(s)}$, and yields a valid token-level assignment (Proposition~\ref{prop:token_validity}). The resulting signals are differentiable with respect to the DTAN parameters (Lemma~\ref{lem:token_differentiable}) and can be integrated into policy-gradient optimization (Corollary~\ref{cor:token_rl}).
\vspace{-0.2cm}
\paragraph{Joint optimization with RLHF.}
DTAN is not trained during reward-model fitting; it is optimized jointly with the policy during the RLHF phase. The decomposed token-level signals are used to construct fine-grained advantage estimates for policy optimization, while the underlying reward model remains fixed. Let $A_{s,t}$ denote the advantage associated with token $t$ in sentence $s$. For PPO-style optimization, the policy objective can be written as
\begin{equation}
\label{eq:loss_rlhf}
    \mathcal{L}_{\mathrm{RLHF}}(\theta) = - \mathbb{E}\big[ A_{s,t}\, \log \pi_\theta(y_{s,t} \mid x, y_{s,<t}) \big].
\end{equation}
Through repeated policy updates, DTAN learns allocation weights that better guide policy improvement without requiring token-level supervision. This design enables bounded token-level refinement while remaining compatible with standard RLHF training pipelines.

%% file: chapters/experiment.tex
\vspace{-0.25cm}
\section{Experiment}
\label{Experiment}
\vspace{-0.15cm}
\subsection{Experimental Setup}
\vspace{-0.15cm}

We use Gemma-2-9B (Team et al., 2024) as the policy model and RM-Gemma-2B (Dong et al., 2023) as the reward model. HH-RLHF (Bai et al., 2022) serves as the primary training and analysis dataset under a standard PPO-based RLHF pipeline. To assess robustness and cross-dataset generalization, we further evaluate the trained policies zero-shot on AdvBench, RealToxicityPrompts, ToxiGen, and SafeRLHF in Appendix~\ref{app:supplement_ex}. Implementation details are deferred to Appendix~\ref{app:exper_details}. All experiments are conducted on a single A100-SXM4-80GB GPU.

\begin{figure*}[!t]
  \centering
  \includegraphics[width=1\columnwidth]{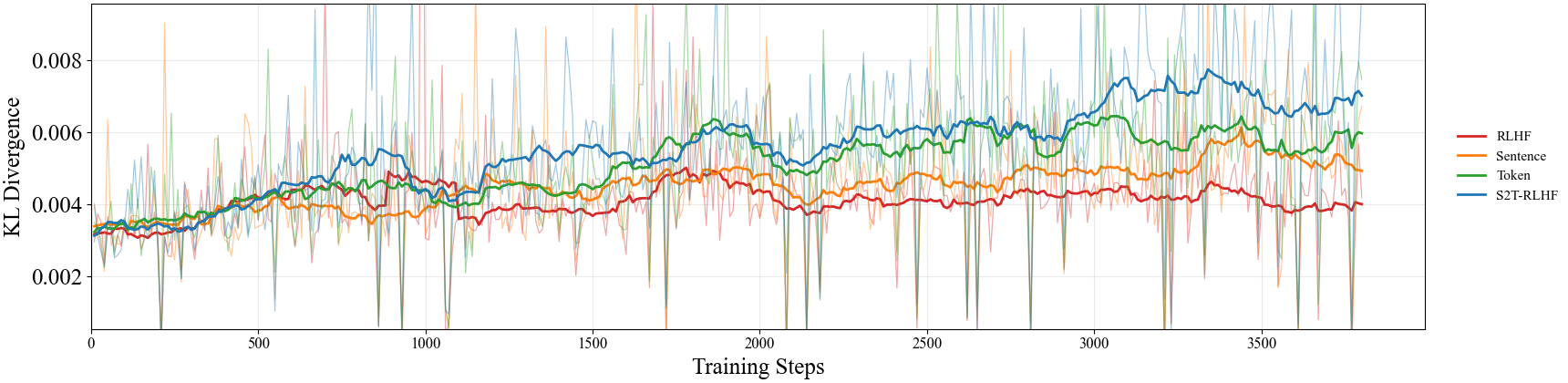}
  \vspace{-0.25cm}
  \caption{KL divergence to the reference policy during training.}
  \label{fig:kl}
  \vskip -0.1in
\end{figure*}

\begin{figure*}[!t]
  \centering
  \includegraphics[width=1\columnwidth]{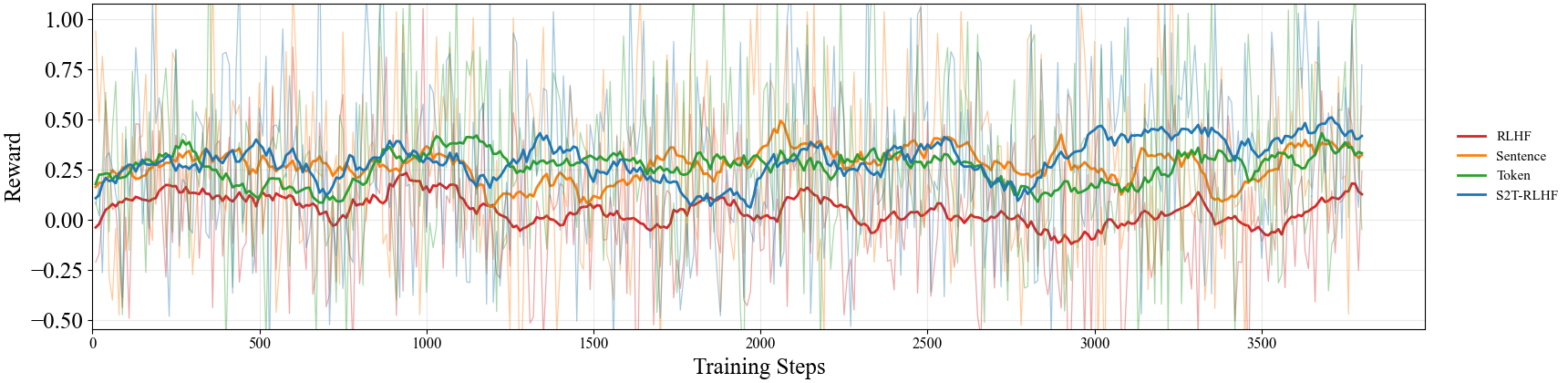}
  \vspace{-0.25cm}
  \caption{Reward trajectory during training.}
  \label{fig:reward}
  \vskip -0.1in
\end{figure*}

\begin{figure*}[!t]
  \centering
  \includegraphics[width=1\columnwidth]{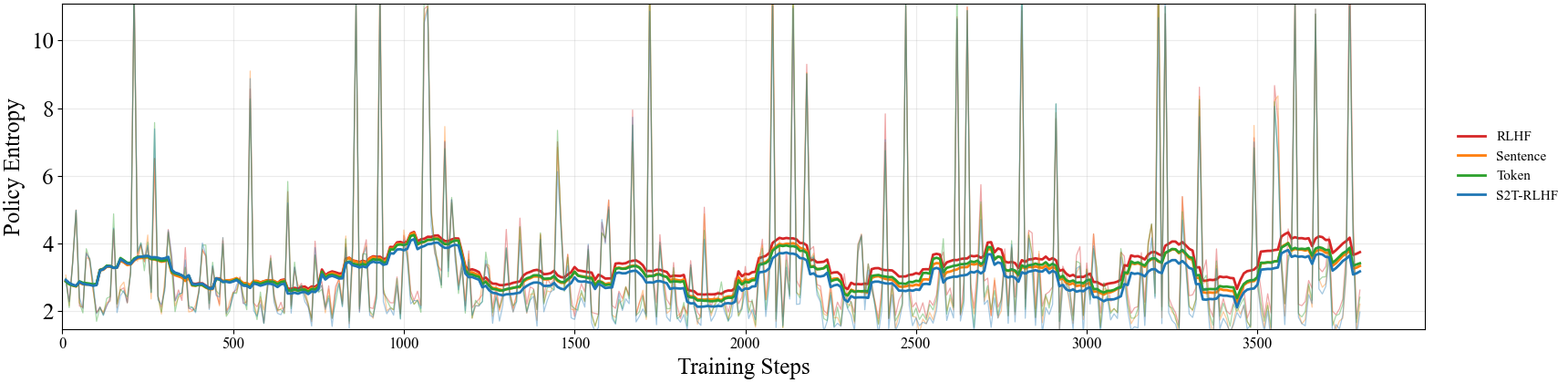}
  \vspace{-0.25cm}
  \caption{Policy entropy evolution during training.}
  \label{fig:entropy}
  \vspace{-0.65cm}
\end{figure*}
\vspace{-0.25cm}
\subsection{Main Results on Training Dynamics and Stability}
\vspace{-0.25cm}

We evaluate three structural variants within the S2T framework:
(i) \emph{sentence-only allocation}, which assigns reward at the sentence level and distributes it uniformly within each sentence;
(ii) \emph{token-only refinement}, which applies bounded token-level reweighting under uniform sentence-level allocation; and
(iii) the full \emph{S2T hierarchy}, which combines sentence-level allocation with bounded token-level refinement.
For reference, we also include standard sequence-level RLHF, where a single scalar reward is propagated to token-level updates under vanilla PPO.

We characterize training dynamics using three complementary signals: KL divergence from the reference policy, policy entropy, and the mean reward trajectory over training steps. KL divergence reflects the magnitude and accumulation of policy updates, entropy captures changes in exploration and potential mode collapse, and the mean reward trajectory highlights temporal fluctuations in the learning signal. We visualize representative cases in Figures~\ref{fig:kl}--\ref{fig:entropy}. Solid lines show moving-average smoothed trajectories, and faint lines indicate raw measurements.

Figure~\ref{fig:kl} shows that all methods exhibit comparable KL divergence in the early stage, indicating similar update scales when training is most fragile. In the mid-to-late stage, S2T progressively departs further from the reference policy than the other variants, yielding the ordering S2T $>$ Token $>$ Sentence $>$ RLHF in KL. Importantly, this larger KL does not coincide with instability. Figure~\ref{fig:reward} shows the same ordering in final mean reward, suggesting that the additional policy drift induced by S2T translates into consistent reward gains while avoiding collapse. Figure~\ref{fig:entropy} further indicates that entropy remains in a comparable range across methods, with no signs of degenerate collapse. S2T exhibits slightly lower late-stage entropy, consistent with a more confident policy while maintaining stable optimization. Finally, the sentence-level variant already provides a favorable trade-off in the early-to-mid stage, achieving higher reward at smaller KL than token-level refinement, supporting our main claim that semantically aligned credit assignment improves reward efficiency and stabilizes learning.

\vspace{-0.15cm}
\textbf{Phase-wise stability metrics.} To strengthen the trajectory-level evidence in Figures~\ref{fig:kl}--\ref{fig:entropy}, we report phase-wise quantitative stability metrics aggregated over three random seeds in Appendix~\ref{app:stability_metrics_default}. We compute summary statistics over an early window (0--1000 steps) and a late window (5000--6000 steps), enabling a controlled comparison of optimization stability across phases. Consistent with the trends in the trajectories, S2T-RLHF exhibits the lowest volatility of the optimization signals and avoids the larger policy drift observed for token-only refinement.
\vspace{-0.15cm}

\textbf{Early-stage stress tests.} In addition, since RLHF failures often concentrate in the early stage, we conduct stress tests under more aggressive learning rates and KL settings and analyze the resulting early-stage dynamics in Appendix~\ref{app:early_stage}. The stress tests highlight that token-only refinement is more prone to early-stage instability under aggressive settings, whereas sentence-level allocation remains comparatively stable. By combining sentence-level allocation with bounded within-sentence refinement, S2T attenuates the instability induced by fine-grained credit while retaining its late-stage optimization benefits. Overall, the sentence-level variant already improves early-stage stability, whereas token-level refinement mainly affects late-stage optimization behavior. By combining both, S2T yields more consistent training dynamics together with stronger final rewards than using either granularity in isolation.
\vspace{-0.15cm}

\textbf{Computational cost.}
We provide a detailed runtime analysis in Appendix~\ref{app:compute}. S2T-RLHF introduces additional computation mainly from perturbation-based sentence influence estimation, while the bargaining solver and DTAN refinement are lightweight.
\vspace{-0.25cm}
\subsection{Preference Alignment Performance}
\vspace{-0.15cm}
We evaluate preference alignment performance using pairwise comparisons, focusing on whether sentence-to-token credit assignment preserves alignment quality while improving training stability.
\begin{table*}[!t]
  \centering
  \small

  \begin{minipage}[t]{0.48\textwidth}
    \centering
    \caption{Pairwise preference evaluation under a GPT-4o-mini evaluator.}
    \label{tab:winrate_judge}
    \begin{sc}
      \begin{tabular}{l c @{}c@{} c c}
        \toprule
        Method & Win (\%) $\uparrow$ &  & Lose (\%) $\downarrow$ & Tie (\%) \\
        \midrule
        vs.\ RLHF & \textbf{53.0} & \scriptsize$>$ & 36.5 & 10.5 \\
        vs.\ ABC  & \textbf{51.5} & \scriptsize$>$ & 38.0 & 10.4 \\
        vs.\ SCAR & \textbf{48.1} & \scriptsize$>$ & 43.1 & 8.8  \\
        \bottomrule
      \end{tabular}
    \end{sc}
  \end{minipage}
  \hfill
  \begin{minipage}[t]{0.48\textwidth}
    \centering
    \caption{Pairwise preference evaluation under an RM-Gemma-7B evaluator.}
    \label{tab:winrate_rm}
    \begin{sc}
      \begin{tabular}{l c @{}c@{} c c}
        \toprule
        Method & Win (\%) $\uparrow$ &  & Lose (\%) $\downarrow$ & Tie (\%) \\
        \midrule
        vs.\ RLHF & \textbf{51.1} & \scriptsize$>$ & 39.5 & 9.4 \\
        vs.\ ABC  & \textbf{48.7} & \scriptsize$>$ & 43.8 & 7.5 \\
        vs.\ SCAR & \textbf{45.8} & \scriptsize$>$ & 44.1 & 10.1  \\
        \bottomrule
      \end{tabular}
    \end{sc}
  \end{minipage}

  \vspace{-0.2in}
\end{table*}

\vspace{-0.25cm}
\paragraph{Baselines.}
We compare S2T-RLHF with three representative approaches:
\textbf{RLHF}~\cite{ouyang2022training}, which applies sequence-level scalar rewards;
\textbf{ABC}~\cite{chan2024dense}, which performs dense token-level credit assignment via attention-based allocation;
and \textbf{SCAR}~\cite{cao2025scar}, which decomposes sequence-level rewards into token-level signals using Shapley-value-based marginal contributions.

\vspace{-0.25cm}
\paragraph{Evaluation Protocol.}
We perform pairwise preference evaluation using two complementary evaluators. First, we adopt an \emph{LLM-as-a-Judge}~\cite{zheng2023judging} protocol based on \textbf{GPT-4o-mini}~\cite{hurst2024gpt,achiam2023gpt} to capture large-scale and consistent preference trends. Second, we evaluate the same model pairs using reward-model-based preferences from \textbf{RM-Gemma-7B}~\cite{dong2023raft} to assess robustness with respect to the choice of evaluator. All methods are trained under identical optimization budgets and compared using the same prompts and evaluation procedure.
\vspace{-0.3cm}

\paragraph{Results and Analysis.}
Table~\ref{tab:winrate_judge} reports win--tie--lose statistics under an \emph{LLM-as-a-Judge} protocol. Across all baselines, S2T-RLHF attains a consistent win advantage (Win $>$ Lose), indicating that sentence-to-token credit assignment preserves strong preference alignment relative to RLHF, ABC, and SCAR. Table~\ref{tab:winrate_rm} reports the same comparison using a learned reward model (RM-Gemma-7B) as the evaluator. While absolute win rates shift across evaluators, the conclusion is unchanged: S2T-RLHF maintains a win advantage over each baseline (Win $>$ Lose), suggesting robustness to the choice of preference signal. Additional evaluations with Qwen-based and larger Gemma-based evaluators are provided in Appendix~\ref{app:model_rm_robustness}, further supporting the robustness of the preference-alignment results across evaluator families and scales.

\textbf{Cross-dataset preference and reward dynamics.} Additional pairwise preference results across datasets are reported in Appendix~\ref{app:cross}. Across most settings, S2T-RLHF maintains a consistent win advantage over the baselines, reinforcing that the proposed stability-oriented decomposition does not sacrifice preference alignment. Appendix~\ref{app:train_time_re} also reports training-time reward trajectories, where S2T-RLHF achieves higher rewards than baseline methods, highlighting more effective utilization of the reward signal.

\textbf{Failure case of fine-grained decomposition.} In addition to these alignment results, we observe a training collapse phenomenon for fine-grained token-level decomposition methods under more aggressive optimization. As illustrated in Appendix~\ref{app:case}, such overly fine-grained reward decomposition can amplify optimization noise and lead to unstable training behavior at larger learning rates.

Overall, preference evaluations verify that a stability-oriented reward decomposition achieves more stable training while maintaining strong alignment performance.

\vspace{-0.2cm}
\subsection{Ablation Study}
\vspace{-0.3cm}
Tables~\ref{tab:ablation_judge} and~\ref{tab:ablation_rm} report pairwise preference outcomes for ablations of the S2T hierarchy under an \emph{LLM-as-a-Judge} protocol and a reward-model-based evaluator, respectively. Across evaluators, S2T-RLHF is consistently preferred over all granularity variants. This indicates that combining sentence-level allocation with bounded within-sentence refinement yields the most effective overall credit-assignment structure among the tested designs. Mechanism-swap ablations in Appendix~\ref{app:mechanism_swap} further show that the gains come from both the two-stage hierarchy and the specific Nash/DTAN allocation mechanisms.
\vspace{-0.3cm}
\paragraph{Discussion.}
The results suggest that the benefit of reward decomposition in RLHF should not be understood solely as a consequence of finer credit assignment. Instead, what matters is whether the reward signal is decomposed at a granularity that matches the semantic structure of language and remains stable under policy optimization. This distinction helps explain why token-only refinement can improve late-stage optimization but is more vulnerable to early-stage instability, whereas sentence-level allocation provides a more robust intermediate abstraction. From this perspective, S2T-RLHF is not merely a denser reward-shaping method, but an instance of a broader granularity-aware reward design principle for preference-based RLHF. We provide further theoretical discussion in Appendices~\ref{app:discussion}--\ref{app:end_to_end}, which analyze the instability of sequence-level RLHF and connect it to the two-stage design of S2T-RLHF.

\begin{table*}[!t]
  \centering
  \small

  \begin{minipage}[t]{0.48\textwidth}
    \centering
    \caption{Ablation preference evaluation under an \emph{LLM-as-a-Judge} evaluator.}
    \label{tab:ablation_judge}
    \begin{sc}
      \begin{tabular}{l c @{}c@{} c c}
        \toprule
        Method & Win (\%) $\uparrow$ &  & Lose (\%) $\downarrow$ & Tie (\%) \\
        \midrule
        vs.\ RLHF & \textbf{53.0} & \scriptsize$>$ & 36.5 & 10.5 \\
        vs.\ Sent.  & \textbf{54.5} & \scriptsize$>$ & 35.4 & 10.1 \\
        vs.\ Token & \textbf{48.6} & \scriptsize$>$ & 41.6 & 9.8  \\
        \bottomrule
      \end{tabular}
    \end{sc}
  \end{minipage}
  \hfill
  \begin{minipage}[t]{0.48\textwidth}
    \centering
    \caption{Ablation preference evaluation under a reward-model evaluator.}
    \label{tab:ablation_rm}
    \begin{sc}
      \begin{tabular}{l c @{}c@{} c c}
        \toprule
        Method & Win (\%) $\uparrow$ &  & Lose (\%) $\downarrow$ & Tie (\%) \\
        \midrule
        vs.\ RLHF & \textbf{51.1} & \scriptsize$>$ & 39.5 & 9.4 \\
        vs.\ Sent.  & \textbf{49.5} & \scriptsize$>$ & 44.6 & 5.9 \\
        vs.\ Token & \textbf{52.2} & \scriptsize$>$ & 43.5 & 4.3  \\
        \bottomrule
      \end{tabular}
    \end{sc}
  \end{minipage}

  \vskip -0.16in
\end{table*}

%% file: chapters/conclusion.tex
\vspace{-0.3cm}
\section{Conclusion}
\label{Conclusion}
\vspace{-0.3cm}

This work revisits reward design in RLHF from the perspective of semantic granularity and training stability.
Through extensive empirical analysis, we show that RLHF instability is not an implementation artifact, but a structural consequence of driving token-level optimization with low-resolution, sequence-level rewards.
Our results challenge the common assumption that increasingly fine-grained credit assignment is inherently beneficial, and instead highlight the importance of granularity-aware reward design that prioritizes semantic structure and stable learning dynamics.
Within this perspective, we propose S2T-RLHF as a concrete instantiation, combining inter-sentence bargaining with lightweight token-level refinement to enable end-to-end credit assignment without modifying the reward model or requiring token-level supervision.

%% file: appendix/app_preliminarie.tex
\section{Preliminaries}
\label{app:preliminaries}

\paragraph{Reinforcement Learning from Human Feedback.}
Reinforcement learning from human feedback (RLHF) is a framework for aligning pretrained language models with human preferences. Let $\pi_\theta(y\mid x)$ denote a conditional language model that generates a response $y = (t_1,\dots,t_T)$ given a prompt $x$. A reward model $R_\phi(x,y) \in \mathbb{R}$ is trained on human preference data to assign a scalar quality score to each complete response. RLHF then optimizes the model parameters $\theta$ to maximize the expected reward under $R_\phi$,  typically with a KL-divergence regularization term that prevents $\pi_\theta$ from deviating too far from a reference model $\pi_{\mathrm{ref}}$~\cite{ouyang2022training}:

\begin{equation}
  \max_{\theta}\;\mathbb{E}_{y \sim \pi_\theta} \big[ R_\phi(x,y) \big] - \beta\,\mathrm{KL}\!\big(\pi_\theta \,\|\, \pi_{\mathrm{ref}}\big).
\end{equation}

where $\beta > 0$ controls the trade-off between reward maximization and divergence from the reference model. In standard implementations, the reward model evaluates responses holistically, yielding a single global sequence-level signal that does not distinguish the contributions of individual tokens.

\paragraph{Nash Bargaining.}
Nash bargaining~\cite{muthoo1999bargaining, navon2022multi, zeng2024fairness} is a classical solution concept in cooperative game theory that prescribes how to divide joint gains among multiple parties in a manner that is both fair and Pareto-efficient. Consider $n$ players with utilities $u_1,\dots,u_n$ and disagreement points $d_1,\dots,d_n$, where $d_i$ represents the utility player $i$ receives if no agreement is reached. Let $\mathcal{U} \subseteq \mathbb{R}^n$ denote the feasible set of utility vectors. The Nash bargaining solution is defined as:
\begin{equation}
  u^\star \in \arg\max_{u \in \mathcal{U},\, u_i \ge d_i} \sum_{i=1}^n \log \big(u_i - d_i\big).
  \label{eq:nash-bargaining}
\end{equation}
This solution is characterized by symmetry, scale invariance, Pareto efficiency, and independence of irrelevant alternatives. It provides a principled framework for allocating utility among interacting components.

\paragraph{Dirichlet Distribution.}
The Dirichlet distribution is a probability distribution defined over the probability simplex, consisting of vectors
$p = (p_1,\dots,p_K)$ satisfying $p_k \ge 0$ and $\sum_{k=1}^K p_k = 1$. 
It is parameterized by a concentration vector $\boldsymbol{\alpha} = (\alpha_1,\dots,\alpha_K) \in (0,\infty)^K$. We write $p \sim \mathrm{Dir}(\boldsymbol{\alpha})$ if $p$ has density
\begin{equation}
    \label{eq:dirichlet}
    \mathrm{Dir}(p ; \boldsymbol{\alpha})
    = \frac{\Gamma\!\left(\sum_{k=1}^K \alpha_k\right)}
         {\prod_{k=1}^K \Gamma(\alpha_k)}
    \prod_{k=1}^K p_k^{\alpha_k - 1}.
\end{equation}
The mean of the Dirichlet distribution is given by
$\mathbb{E}[p_k] = \alpha_k / \alpha_0$ for all $k$, where
$\alpha_0 = \sum_{k=1}^K \alpha_k$ denotes the total concentration parameter,
which controls how peaked or diffuse the distribution is over the simplex.
As a distribution over normalized proportions, the Dirichlet distribution is widely used to model uncertainty in categorical probabilities and to represent relative weights over multiple outcomes, especially as a conjugate prior for categorical or multinomial variables.

%% file: appendix/app_proof.tex
\newpage 
\section{Proofs}
\label{app:proofs}

\begin{lemma}[Stationarity in open domains]
\label{lem:open_stationary}
Let $\mathcal{D}\subset \mathbb{R}^n$ be open and $\phi:\mathcal{D}\to\mathbb{R}$ 
continuously differentiable. Then $x^\star\in\mathcal{D}$ satisfies 
$\nabla \phi(x^\star)=0$ if and only if $x^\star$ is a stationary point 
of the unconstrained problem $\min_{x\in\mathcal{D}} \phi(x)$.
\end{lemma}
\begin{proof}
Since $\mathcal{D}$ is open, there are no boundary constraints, and 
first-order stationarity reduces to $\nabla \phi(x^\star)=0$.
\end{proof}

\paragraph{Idealized analysis vs. practical safeguard.}
The results in this section analyze the idealized fixed-point equation $M\beta=\beta^{\odot-1}$ (i.e., without the $\varepsilon$-safeguard) to elucidate its Nash-style structure.
Our implementation uses the safeguarded utilities $u=\max(M\beta,\varepsilon)$ to ensure well-defined logarithms and reciprocal updates; see Appendix~\ref{app:positivity} for details.

\begin{proposition}[Stationarity of the fixed-point equilibrium]
\label{prop:kkt}
Any positive fixed point $\beta^\star \in \mathbb{R}^K_{++}$ satisfying  $M\beta^\star = (\beta^\star)^{\odot-1}$ is a stationary point of the potential function
\begin{equation}
\label{eq:phi_def}
    \phi(\beta) := \frac{1}{2}\beta^\top M\beta - \sum_{i=1}^K \log \beta_i.
\end{equation}
Moreover, with $u = M\beta^\star$, the fixed point satisfies the Nash-style equalization condition $u_i \beta^\star_i = 1$ for all $i$.
\end{proposition}

\begin{proof}
We first establish stationarity of $\phi$. Since $\phi$ is continuously differentiable on the open domain $\mathbb{R}^K_{++}$. The gradient is
\begin{equation}
\label{eq:grad_phi}
    \nabla \phi(\beta) = M\beta - \beta^{\odot-1},
\end{equation}
where we used the symmetry of $M$ for the quadratic term and $\nabla(-\sum_i \log \beta_i) = -\beta^{\odot-1}$ for the logarithmic barrier.

For any fixed point $\beta^\star$ satisfying $M\beta^\star = (\beta^\star)^{\odot-1}$, 
we have $\nabla \phi(\beta^\star) = 0$. Since $\mathbb{R}^K_{++}$ is open (no boundary constraints), this is equivalent to the first-order stationarity condition of $\phi$ (Lemma~\ref{lem:open_stationary}).

For the Nash-style characterization, let $u:= M\beta^\star$.  The fixed-point equation $M\beta^\star = (\beta^\star)^{\odot-1}$ directly implies $u = (\beta^\star)^{\odot-1}$, yielding 
$u_i \beta^\star_i = 1$ for all $i$. This equalization condition characterizes the marginal utility balance in the Nash bargaining framework: since $\nabla_u \sum_i \log u_i = u^{\odot-1}$, 
the fixed point ensures $u^{\odot-1} = \beta^\star$, meaning marginal utilities in the $u$-space are proportional to the bargaining coefficients.
\end{proof}

\begin{corollary}[Existence of positive fixed points]
\label{cor:existence}
If the interaction matrix $M = V^\top V$ is positive definite, then the 
fixed-point equation $M\beta = \beta^{\odot -1}$ admits at least one 
solution $\beta^\star \in \mathbb{R}^K_{++}$.
\end{corollary}

\begin{proof}
Consider the potential function 
$\phi(\beta) = \frac{1}{2}\beta^\top M\beta - \sum_{i=1}^K \log\beta_i$ 
on $\mathbb{R}^K_{++}$. 

Since $M \succ 0$, we have:
\begin{itemize}
    \item The quadratic term $\frac{1}{2}\beta^\top M\beta$ is strictly convex 
    and grows as $\|\beta\|^2$ when $\|\beta\| \to \infty$.
    \item The logarithmic barrier $-\sum_i \log\beta_i$ tends to $+\infty$ 
    as any $\beta_i \to 0^+$.
\end{itemize}

Therefore, $\phi$ is coercive on $\mathbb{R}^K_{++}$ and attains its minimum 
at some $\beta^\star \in \mathbb{R}^K_{++}$. At this minimizer, the first-order 
condition $\nabla\phi(\beta^\star) = 0$ yields $M\beta^\star = (\beta^\star)^{\odot -1}$.
\end{proof}

\begin{lemma}[Well-definedness and positivity under safeguarded iteration]
\label{lem:positivity}
Let $\varepsilon>0$ and define the safeguarded utility $u(\beta):=\max(M\beta,\varepsilon)$ elementwise.
Then the damped iteration
\begin{equation}
    \beta^{(t+1)} = (1-\eta)\beta^{(t)} + \eta\, u(\beta^{(t)})^{\odot-1},
    \quad \eta\in(0,1],
\end{equation}
is well-defined and preserves positivity: if $\beta^{(0)}\in\mathbb{R}^K_{>0}$, then $\beta^{(t)}\in\mathbb{R}^K_{>0}$ for all $t\ge 0$.
\end{lemma}

\begin{proof}
By construction, $u(\beta)\ge \varepsilon$ elementwise, hence $u(\beta)^{\odot-1}>0$ elementwise.
Since $\beta^{(t+1)}$ is a convex combination of two strictly positive vectors $\beta^{(t)}$ and $u(\beta^{(t)})^{\odot-1}$, positivity follows by induction.
\end{proof}

\begin{remark}[Practical implications]
Corollary~\ref{cor:existence} guarantees that when $M$ is positive definite (which holds when the influence vectors $\{v_i\}_{i=1}^K$ are linearly independent), the Nash-style bargaining equilibrium exists. Lemma~\ref{lem:positivity} shows that the safeguarded iteration is well-defined and preserves positivity, avoiding division-by-zero or negative weights during optimization.
\end{remark}

\begin{proposition}[Validity of token-level reward decomposition]
\label{prop:token_validity}
    Let $\alpha_{s,t} > 0$ for $t=1,\dots,L_s$, and define $c_{s,t} = \frac{\alpha_{s,t}}{\sum_{k=1}^{L_s} \alpha_{s,k}}$, while $r_{s,t} = R^{(s)} c_{s,t}$. Then $\{c_{s,t}\}_{t=1}^{L_s}$ lies on the probability simplex, and the induced token-level rewards satisfy $\sum_{t=1}^{L_s} r_{s,t} = R^{(s)}$.
\end{proposition}

\begin{proof}
Since $\alpha_{s,t}>0$, we have $c_{s,t}\ge 0$ and $\sum_{t=1}^{L_s} c_{s,t}
= \frac{\sum_{t=1}^{L_s}\alpha_{s,t}}{\sum_{k=1}^{L_s}\alpha_{s,k}} = 1$; hence $c_s$ lies on the simplex, and
$\sum_{t=1}^{L_s} r_{s,t} = R^{(s)}\sum_{t=1}^{L_s} c_{s,t} = R^{(s)}$.
\end{proof}

\begin{lemma}[Differentiability of Dirichlet-based allocation]
\label{lem:token_differentiable}
    Assume $\alpha_{s,t} = \mathrm{softplus}(f_\phi(h_{s,t}))$ for all $t$.
    Then the mapping $\phi \mapsto c_{s,t}$ is continuously differentiable.
    Moreover, letting $S_s = \sum_{k=1}^{L_s} \alpha_{s,k}$, the partial derivatives satisfy:
    \begin{equation}
        \frac{\partial c_{s,t}}{\partial \alpha_{s,k}}=
        \begin{cases}
            \frac{S_s - \alpha_{s,t}}{S_s^2}, & k=t, \\[4pt]
            -\frac{\alpha_{s,t}}{S_s^2}, & k \neq t.
        \end{cases}
    \end{equation}
\end{lemma}

\begin{proof}
    Since $\alpha_{s,t}=\mathrm{softplus}(f_\phi(h_{s,t}))$ and $\mathrm{softplus}(\cdot)$ is smooth, each $\alpha_{s,t}$ is continuously differentiable with respect to $\phi$. Consequently, $S_s=\sum_{k=1}^{L_s}\alpha_{s,k}$ is continuously differentiable and strictly positive, which implies that $c_{s,t}=\alpha_{s,t}/S_s$ is continuously differentiable by composition.
    
    To compute the partial derivatives, consider $c_{s,t}=\alpha_{s,t}/S_s$ with $S_s=\sum_{j=1}^{L_s}\alpha_{s,j}$. Applying the quotient rule yields:
    \begin{equation}
        \frac{\partial c_{s,t}}{\partial \alpha_{s,k}}=\frac{\frac{\partial \alpha_{s,t}}{\partial \alpha_{s,k}}\, S_s -\alpha_{s,t}\, \frac{\partial S_s}{\partial \alpha_{s,k}}}{S_s^2}.
    \end{equation}
    If $k=t$, then $\frac{\partial \alpha_{s,t}}{\partial \alpha_{s,k}}=1$ and $\frac{\partial S_s}{\partial \alpha_{s,k}}=1$, giving $\frac{S_s-\alpha_{s,t}}{S_s^2}$. If $k\neq t$, then $\frac{\partial \alpha_{s,t}}{\partial \alpha_{s,k}}=0$ and $\frac{\partial S_s}{\partial \alpha_{s,k}}=1$, yielding $-\frac{\alpha_{s,t}}{S_s^2}$.
\end{proof}

\begin{corollary}[Compatibility with policy-gradient methods]
\label{cor:token_rl}
The token-level rewards $\{r_{s,t}\}$ defined in Proposition~\ref{prop:token_validity} constitute a valid reward signal for policy-gradient optimization.
\end{corollary}

\begin{proof}
By Proposition~\ref{prop:token_validity}, the decomposition satisfies $\sum_{t=1}^{L_s} r_{s,t}=R^{(s)}$ with $r_{s,t}=R^{(s)}c_{s,t}$. Thus $\{r_{s,t}\}$ can be used as step-wise rewards in a standard policy-gradient objective by treating each token generation as a decision step and computing returns/advantages from these rewards in the usual way.
\end{proof}

\begin{lemma}[Scale invariance of token allocation]
\label{lem:token_scale_invariant}
    Let $c_{s,t} = \alpha_{s,t} / \sum_k \alpha_{s,k}$. For any scalar $\lambda > 0$, define $\alpha'_{s,t} = \lambda \alpha_{s,t}$. Then the resulting allocation satisfies $c'_{s,t} = c_{s,t}$ for all $t$.
\end{lemma}

\begin{proof}
By definition,
\begin{equation}
    c'_{s,t} = \frac{\alpha'_{s,t}}{\sum_{k=1}^{L_s}\alpha'_{s,k}} = \frac{\lambda \alpha_{s,t}}{\sum_{k=1}^{L_s}\lambda \alpha_{s,k}} = \frac{\lambda \alpha_{s,t}}{\lambda \sum_{k=1}^{L_s}\alpha_{s,k}} = \frac{\alpha_{s,t}}{\sum_{k=1}^{L_s}\alpha_{s,k}} = c_{s,t},
\end{equation}
which holds for all $t$.
\end{proof}

\begin{lemma}[Monotonicity of token contributions]
\label{lem:token_monotonicity}
    Fix $\{\alpha_{s,k}\}_{k \neq t}$. Then $c_{s,t}$ is strictly increasing in $\alpha_{s,t}$ and strictly decreasing in $\alpha_{s,k}$ for any $k \neq t$.
\end{lemma}

\begin{proof}
Let $S_s=\sum_{j=1}^{L_s}\alpha_{s,j}$ and recall $c_{s,t}=\alpha_{s,t}/S_s$.
Fixing $\{\alpha_{s,k}\}_{k\neq t}$, we have
\begin{equation}
    \frac{\partial c_{s,t}}{\partial \alpha_{s,t}} = \frac{S_s-\alpha_{s,t}}{S_s^2} = \frac{\sum_{j\neq t}\alpha_{s,j}}{S_s^2} >0,
\end{equation}
since $\alpha_{s,j}>0$ implies $\sum_{j\neq t}\alpha_{s,j}>0$. For any $k\neq t$,
\begin{equation}
    \frac{\partial c_{s,t}}{\partial \alpha_{s,k}} = -\frac{\alpha_{s,t}}{S_s^2} <0.
\end{equation}
Therefore $c_{s,t}$ is strictly increasing in $\alpha_{s,t}$ and strictly decreasing in each $\alpha_{s,k}$ for $k\neq t$.
\end{proof}

%% file: appendix/app_method.tex
\newpage
\section{Method Supplement}
\label{app:method_supp}

\subsection{Stage-I Bargaining Solver: Well-definedness and Diagnostics}
\label{app:positivity}
Although $M=V^\top V$ is positive semidefinite by construction, it is not guaranteed to be elementwise nonnegative; consequently, some entries of $M\beta$ may be non-positive for $\beta\in\mathbb{R}^K_{>0}$.
To ensure the Nash-style objective $\sum_i \log u_i$ and the reciprocal fixed-point updates are well-defined, we apply a standard positivity safeguard:
\begin{equation}
u=\max(M\beta,\varepsilon)\quad \text{(elementwise)},\qquad \varepsilon>0.
\end{equation}
In our implementation, we additionally monitor the fraction of non-positive entries in $M\beta$ during the iteration; the safeguard is used as a numerical guardrail and is rarely activated in practice.
Lemma~\ref{lem:positivity} (Appendix~\ref{app:proofs}) shows that the safeguarded iteration is well-defined and preserves $\beta^{(t)}\in\mathbb{R}^K_{>0}$ for all $t$.

\subsection{Advantage Construction under Reward Decomposition}
\label{app:advantage}

In the main text, we present a conceptual formulation of sentence-to-token reward
decomposition based on semantic allocation.
This appendix describes how this decomposition is instantiated in practice under
PPO-style RLHF training, where optimization is driven by advantage estimates rather
than explicit token-level returns.
We emphasize that the underlying PPO objective remains unchanged; the proposed
decomposition refines the construction of the advantage signal used for credit
assignment, together with additional engineering stabilizations for training
robustness.

\subsubsection{Sequence-Level Advantage}

Given a generated response with sequence-level reward $R^{(\mathrm{seq})}$, we first construct a standard PPO-style sequence-level advantage
\begin{equation}
A^{(\mathrm{seq})} = R^{(\mathrm{seq})} - b,
\end{equation}
where $b$ denotes a PPO baseline, such as a value-function estimate or a batch-wise baseline. This scalar advantage serves as the global optimization signal and defines the overall scale of policy updates.

\subsubsection{Token-Level Advantage under Reward Decomposition}

As described in the main text, sentence-level bargaining decomposes the sequence-level reward $R^{(\mathrm{seq})}$ into normalized sentence allocation masses $\{\omega_s\}_{s=1}^S$, where $\omega_s \ge 0$ and $\sum_{s=1}^S \omega_s = 1$. The corresponding sentence-level reward is $R^{(s)}=\omega_s R^{(\mathrm{seq})}$, which preserves the total reward mass: $\sum_{s=1}^S R^{(s)} = R^{(\mathrm{seq})}$. Within each sentence, token-level allocation further produces normalized contribution scores $c_{s,t}$, with $c_{s,t}\ge 0$ and $\sum_{t=1}^{L_s} c_{s,t}=1$. For practical advantage modulation, we define the token-level allocation weight as
\begin{equation}
	w_{s,t} = \omega_s \cdot c_{s,t},
\end{equation}
where $w_{s,t} \ge 0$ represents the relative allocation mass assigned to token $t$ in sentence $s$. The signed optimization direction remains controlled by the sequence-level advantage $A^{(\mathrm{seq})}$.

In practice, these quantities are flattened into a single non-negative per-token weight $w_t$ over all generated tokens. Rather than constructing separate token-level returns, we incorporate this sentence-to-token decomposition by reweighting the sequence-level advantage. A natural token-level advantage is obtained by proportional modulation:
\begin{equation}
	A_t^{\mathrm{raw}} = A^{(\mathrm{seq})} \cdot \frac{w_t}{\bar w},
\qquad
\bar w = \frac{1}{T}\sum_{t=1}^T w_t,
\end{equation}
where $T$ denotes the number of generated tokens. Since the weights $\{w_t\}_{t=1}^T$ are non-negative allocation masses, $\bar w$ is well-defined for non-empty generated responses. This construction preserves the expected advantage magnitude,
\begin{equation}
	\mathbb{E}_t\!\left[A_t^{\mathrm{raw}}\right] = A^{(\mathrm{seq})},
\end{equation}
ensuring consistency with the original PPO formulation while enabling fine-grained token-level credit assignment.

\subsubsection{Stabilized Token-Level Advantage with Clipping}

In practice, unconstrained token-level reweighting may introduce high variance due
to extreme allocation scores.
To improve optimization stability, we apply a bounded refinement by mixing the
token-level modulation with the original sequence-level signal:
\begin{equation}
A_t
=
A^{(\mathrm{seq})}
\cdot
\Big[
(1-\lambda) + \lambda \cdot \mathrm{clip}\!\Big(\frac{w_t}{\bar w}\Big)
\Big],
\end{equation}
where $\lambda \in [0,1]$ controls the strength of token-level modulation and
$\mathrm{clip}(\cdot)$ enforces bounded influence.

This formulation preserves the global advantage scale while preventing excessively
fine-grained token-level signals from overriding sentence-level alignment, resulting
in more stable PPO updates without modifying the underlying optimization objective.

\subsection{Semantic Perturbation Details}
\label{app:semantic_perturbation}

For sentence-level influence estimation, we generate local perturbations using a
simple deletion-based procedure. Given a response $x=[s_1,\ldots,s_K]$, we perturb
one sentence at a time while keeping all other sentences fixed. For each sentence
$s_i$, we generate up to $K_p=3$ variants by deleting one token from the sentence.
When spaCy is available, deletion candidates are restricted to content words
(nouns, verbs, adjectives, and adverbs); otherwise, we fall back to whitespace
tokenization and delete one token at a time. Variants that remove the entire
sentence are discarded.

Each variant $s_i^{(k)}$ is inserted back into the original response to obtain
$x^{(i\rightarrow k)}$. We then compute the reward difference
\begin{equation}
    \Delta R_{i,k}=R(x)-R(x^{(i\rightarrow k)}),
\end{equation}
and the corresponding embedding difference
\begin{equation}
    \delta_{i,k}=z_i-z_i^{(k)}.
\end{equation}
The influence vector $v_i$ is estimated by ridge regression:
\begin{equation}
    v_i =
    \arg\min_v
    \sum_k
    \left(\Delta R_{i,k}-\langle v,\delta_{i,k}\rangle\right)^2
    +\lambda\|v\|_2^2 ,
\end{equation}
with $\lambda=10^{-3}$. If no valid perturbation is available for a sentence, we
set its influence vector to zero. This procedure requires no additional model
training or external annotations.

%% file: appendix/app_analysis.tex
\newpage
\section{Discussion}
\label{app:discussion}

Our results suggest that instability in RLHF is not an implementation artifact, but a structural consequence of driving token-level optimization with low-resolution, sequence-level preference signals. While reward decomposition can substantially improve stability, its benefit does not stem from more precise token-level allocation. Instead, it functions as a form of structured regularization that redistributes noisy preference signals in a way that aligns better with the semantic organization of language.

This perspective echoes recent findings in deep reinforcement learning, such as Hindsight-DICE~\cite{velu2023hindsight}, which argue that credit assignment instability often arises from structural properties of the learning objective rather than from estimation noise alone. From this viewpoint, increasingly fine-grained credit assignment is not inherently beneficial. Reward models provide noisy, indirect estimates of human preference rather than ground-truth supervision, and pushing allocation to very fine granularity can amplify spurious correlations and local artifacts. Our empirical results show that such over-refined advantages may degrade both performance and training stability, challenging the common assumption that finer credit assignment necessarily leads to better RLHF.

This observation motivates a granularity-aware view of reward design, where the goal is not maximal allocation precision but stable and semantically grounded learning dynamics.
S2T-RLHF instantiates this principle by allocating reward at the sentence level to preserve high-level semantic coherence, while allowing only limited token-level refinement within sentences.
Importantly, S2T should be viewed as one concrete realization of this design philosophy rather than its final form.

Finally, reward decomposition mitigates but does not eliminate RLHF instability.
It does not correct reward model misspecification, nor does it guarantee robustness under arbitrary preference shifts or adversarial prompts.
Understanding the fundamental stability limits of scalar preference supervision remains an important open question.

\subsection{Sources of Instability in RLHF}
\label{analy:instability}

We analyze the sources of instability in RLHF from a theoretical perspective, focusing on the interaction between sequence-level rewards and token-level policy optimization. Consider an autoregressive language model trained with a PPO-style RLHF objective, which can be written as $\mathcal{L}_{\mathrm{RLHF}}(\theta) = - \mathbb{E}\big[ A_t \log \pi_\theta(y_t \mid x, y_{<t}) \big]$, which matches the formulation in Eq.~\eqref{eq:loss_rlhf}.

In standard RLHF, supervision is provided via a sequence-level reward.
Equivalently, this can be viewed as a terminal reward assigned only at the final generation step:
\begin{equation}
    r_t =
    \begin{cases}
        0, & t < T, \\
        R(y), & t = T,
    \end{cases}
\end{equation}
where $R(y)$ denotes the scalar reward assigned to the complete generated sequence.

In practice, RLHF typically uses a discount factor $\gamma \approx 1$ (or $\gamma = 1$), in which case the token-level advantage can be approximated as
\begin{equation}
    A_t \approx \gamma^{T-t} R(y) - V(s_t),
\end{equation}
where $s_t = (x, y_{<t})$ and $V(s_t)$ denotes the value baseline. This formulation reveals a fundamental structural issue of sequence-level RLHF. Since the same scalar reward $R(y)$ influences the advantage of all tokens along the trajectory, the resulting advantages are highly correlated across time steps and often share the same sign, especially when the value function is imperfect. As a result, token-level policy updates are driven by a global signal that provides little discrimination among individual decisions, leading to high-variance and unstable optimization dynamics.
\paragraph{Directional Alignment of Token-Level Updates.}
\begin{remark}[Advantage Alignment under Terminal Rewards]
\label{rem:adv_alignment}
Consider sequence-level RLHF with a terminal reward $R(y)$ and a discount factor $\gamma \approx 1$.
When the value function is imperfectly estimated, it is reasonable to assume that $|V(s_t)| \ll |R(y)|$ for a large portion of the trajectory. Under this assumption, the token-level advantage can be approximated as
\begin{equation}
    A_t \approx \gamma^{T-t} R(y),
\end{equation}
which implies that the advantages of most tokens along the trajectory share the same sign. As a consequence, a single policy update induces a coherent update direction across all tokens in the sequence:
\begin{equation}
    \Delta \theta \propto \sum_{t=1}^{T} A_t \nabla_\theta \log \pi_\theta(a_t \mid s_t).
\end{equation}
This observation reveals a structural mismatch between token-level optimization and sequence-level supervision. The policy is updated using a global scalar signal that provides limited discrimination among individual decisions. As a result, the optimization dynamics can be dominated by a small subset of tokens or patterns that exert a disproportionate influence on the reward model, thereby increasing the risk of over-optimization and reward hacking.
\end{remark}

\paragraph{Variance Amplification over Long Horizons.}
\begin{remark}[Variance Amplification under Long-Horizon Terminal Rewards]
\label{rem:variance_amplification}
Consider sequence-level RLHF with a terminal reward $R(y)$ and discount factor $\gamma \approx 1$.
Let the token-level advantage be defined as
\begin{equation}
    A_t = \gamma^{T-t} R(y) - V(s_t),
\end{equation}
where $V(s_t)$ denotes the value baseline.
Treating $R(y)$ as a random variable induced by the reward model and $V(s_t)$ as an imperfect baseline estimate, the variance of the advantage admits the decomposition
\begin{equation}
    \mathrm{Var}(A_t) = \mathrm{Var}(\gamma^{T-t} R(y) - V(s_t)) \ge \gamma^{2(T-t)} \mathrm{Var}(R(y)) - 2 \gamma^{T-t} \mathrm{Cov}(R(y), V(s_t)).
\end{equation}

In early stages of training, the covariance term $\mathrm{Cov}(R(y), V(s_t))$ is often limited, so the variance of $A_t$ is dominated by the scaled reward variance.
Under $\gamma \approx 1$, this implies that the advantages associated with earlier tokens exhibit substantially larger variance.
When combined with the intrinsic noise of token-level policy gradients, such variance amplification makes PPO-style updates more sensitive to clipping and KL constraints, thereby increasing the difficulty of stable optimization.
\end{remark}

\paragraph{Length-Dependent Gradient Scaling.}
\begin{remark}[Length-Dependent Gradient Scaling under Terminal Rewards]
\label{rem:length_bias}
Consider sequence-level RLHF with a terminal reward $R(y)$ and discount factor $\gamma \approx 1$.
Assume that the token-level advantages $A_t$ along a trajectory share the same sign, as discussed in Remark~\ref{rem:adv_alignment}. The policy gradient can then be written as:
\begin{equation}
    \nabla_\theta J(\theta) \propto \sum_{t=1}^{T} A_t \nabla_\theta \log \pi_\theta(a_t \mid s_t),
\end{equation}
where $a_t$ denotes the generated token and $s_t$ the corresponding autoregressive state.

Let $g_t = \nabla_\theta \log \pi_\theta(a_t \mid s_t)$.
Under mild regularity assumptions, such as non-vanishing diagonal contributions and non-negative correlations across time steps, the expected squared norm of the gradient satisfies
\begin{equation}
    \mathbb{E}\!\left[\left\| \sum_{t=1}^{T} A_t g_t \right\|^2\right] \quad \text{grows with } T,
\end{equation}
at least linearly under these conditions. This implies that longer sequences exert a systematically larger influence on the optimization dynamics when rewards are assigned at the sequence level.

This length-dependent scaling does not necessarily encourage longer outputs per se. However, it introduces an implicit coupling between sequence length and gradient magnitude, making the optimization process sensitive to variations in sequence length and further contributing to instability in RLHF training.
\end{remark}

Remarks~\ref{rem:adv_alignment}, \ref{rem:variance_amplification}, and \ref{rem:length_bias} together characterize three complementary sources of instability in sequence-level RLHF, arising from directional alignment, variance amplification, and length-dependent gradient scaling, respectively.

\subsection{Stabilizing Effects of Reward Decomposition}
\label{analy:effect}

We now analyze how decomposing a sequence-level reward into localized sentence- and token-level signals alleviates the instability sources identified in Section~\ref{analy:instability}. Our analysis focuses on structural properties of the resulting reward signals, rather than on specific optimization heuristics. Importantly, the proposed decomposition preserves the total reward mass, ensuring that stabilization arises from improved credit assignment rather than implicit reward rescaling.

\paragraph{Breaking Advantage Alignment.}
\begin{remark}[Localized Rewards Reduce Advantage Alignment]
\label{rem:decomp_alignment}
Consider a reward decomposition $\{r_t\}_{t=1}^{T}$ satisfying the conservation property $\sum_{t=1}^{T} r_t = R(y)$. Under the decomposed reward, the token-level advantage can be written as:
\begin{equation}
    A_t^{\mathrm{dec}} = \sum_{k=t}^{T} \gamma^{k-t} r_k - V(s_t).
\end{equation}
When the reward components $r_k$ are supported on localized tokens or sentence segments, the resulting advantages vary across time steps and no longer share a common global scaling factor.

As a consequence, token-level policy updates are driven by heterogeneous advantage signals rather than a single sequence-level scalar. This breaks the coherent alignment of update directions observed under terminal rewards and enables finer-grained credit assignment across the trajectory.
\end{remark}

\paragraph{Reducing Effective Horizon and Variance.}
\begin{remark}[Variance Reduction via Localized Reward Support]
\label{rem:decomp_variance}
Under reward decomposition, the token-level advantage depends only on reward components within the future support of $\{r_k\}_{k \ge t}$.
If the reward support is localized to a window of bounded length, the effective horizon contributing to $A_t^{\mathrm{dec}}$ is substantially shorter than the full sequence length.

As a result, the variance of the decomposed advantage is governed by the variability of local reward components rather than by the variance of a global terminal reward.
This reduction in effective horizon mitigates variance amplification for early tokens and improves the signal-to-noise ratio of token-level policy gradients.
\end{remark}

\paragraph{Decoupling Gradient Scale from Sequence Length.}
\begin{remark}[Length-Decoupled Gradient Scaling under Reward Decomposition]
\label{rem:decomp_length}
Consider the policy gradient induced by a decomposed reward signal,
\begin{equation}
\nabla_\theta J(\theta)
\propto \sum_{t=1}^{T} A_t^{\mathrm{dec}} \nabla_\theta \log \pi_\theta(a_t \mid s_t).
\end{equation}
When reward components are localized, the magnitude of $A_t^{\mathrm{dec}}$ depends primarily on local reward mass rather than on the total sequence length. Consequently, the overall gradient scale becomes less sensitive to variations in sequence length. This decoupling alleviates the implicit coupling between length and optimization pressure induced by terminal rewards, leading to more stable updates across responses of different lengths.
\end{remark}

Together, Remarks~\ref{rem:decomp_alignment}, \ref{rem:decomp_variance}, and \ref{rem:decomp_length} show that reward decomposition alleviates instability in RLHF by improving credit locality, reducing variance amplification, and decoupling gradient scale from sequence length, while preserving the total reward mass.

\subsection{Limitations}
\label{app:limitations}

Although S2T-RLHF improves the stability of preference-based RLHF, it has several limitations. First, the method still relies on a learned sequence-level reward model. Reward decomposition can redistribute the scalar reward into more localized learning signals, but it cannot correct reward-model misspecification, systematic bias, or preference errors in the original reward signal. If the reward model assigns high scores to undesirable responses, S2T-RLHF may still propagate this incorrect preference through the decomposed rewards.

Second, the proposed sentence-to-token hierarchy is most suitable for open-ended generation tasks where responses can be naturally segmented into semantically meaningful sentence-level units. For tasks whose correctness depends on tightly coupled symbolic reasoning, such as mathematical proof generation or code execution, sentence-level decomposition may not fully capture step-level dependencies or executable correctness. In such settings, process-level supervision or verifier-based feedback may be more appropriate.

Third, S2T-RLHF introduces additional computational overhead due to perturbation-based sentence influence estimation. Our runtime analysis in Appendix~\ref{app:compute} shows that this overhead is mainly concentrated in Stage I, while the bargaining solver and DTAN refinement are lightweight. Future work may reduce this cost through more efficient perturbation selection, batched reward-model evaluation, caching, or amortized influence estimation.

Finally, while our experiments cover multiple datasets and stress-test settings, they are still limited to a specific policy model, reward-model family, and PPO-based RLHF pipeline. Extending the analysis to larger models, alternative reward models, other policy optimization algorithms, and more diverse task domains remains an important direction for future work.

%% file: appendix/app_end.tex
\newpage

\section{End-to-End Theoretical Justification of S2T-RLHF}
\label{app:end_to_end}

The analysis in Appendix~\ref{app:discussion} focused on the problem side of RLHF instability. In particular, we showed that when a sequence-level terminal reward is propagated to token-level optimization, the resulting learning dynamics are prone to three structural failure modes: global advantage alignment, long-horizon variance amplification, and length-dependent gradient scaling. We now connect this instability analysis to the design of S2T-RLHF itself. Specifically, we show that the two-stage structure of S2T-RLHF is not an arbitrary composition of sentence-level and token-level modules: Stage I induces sentence-level reward localization, while Stage II introduces bounded within-sentence refinement. Together, these two properties provide an end-to-end justification for why the S2T design mitigates the instability mechanisms identified in Appendix~\ref{app:discussion}.

\begin{lemma}[Sentence-level localization and conservation]
\label{lem:stage1_localization}
Let $\omega_i=\frac{u_i}{\sum_{j=1}^K u_j}$ and $r_i=\omega_i R(x)$ denote
the Stage-I sentence allocation weights and sentence-level rewards,
respectively. Then
\begin{equation}
\omega_i \ge 0 \quad \forall i,
\qquad
\sum_{i=1}^K \omega_i = 1,
\qquad
\sum_{i=1}^K r_i = R(x).
\end{equation}
Thus, Stage I redistributes the sequence-level reward over sentence-level semantic units while preserving the total reward mass.
\end{lemma}

\begin{proof}
By construction, $u_i=\max((M\beta)_i,\epsilon)>0$ for all $i$, so the normalized coefficients
\begin{equation}
\omega_i=\frac{u_i}{\sum_{j=1}^K u_j}
\end{equation}
are well defined, nonnegative, and satisfy
\begin{equation}
\sum_{i=1}^K \omega_i
=
\frac{\sum_{i=1}^K u_i}{\sum_{j=1}^K u_j}
=1.
\end{equation}
Since $r_i=\omega_i R(x)$, it follows immediately that
\begin{equation}
\sum_{i=1}^K r_i
=
\sum_{i=1}^K \omega_i R(x)
=
\Big(\sum_{i=1}^K \omega_i\Big)R(x)
=
R(x).
\end{equation}
Thus, Stage I preserves the total sequence-level reward while inducing a sentence-level allocation.
\end{proof}

\begin{lemma}[Bounded within-sentence refinement]
\label{lem:stage2_bounded_refinement}
For sentence $s$, Stage II defines token-level weights
$c_{s,t}=\alpha_{s,t}/\sum_{k=1}^{L_s}\alpha_{s,k}$ with $\alpha_{s,t}>0$,
and token-level rewards $r_{s,t}=R^{(s)}c_{s,t}$. Hence
$\{c_{s,t}\}_{t=1}^{L_s}$ lies on the probability simplex and preserves
the sentence-level reward mass, i.e.,
\begin{equation}
c_{s,t}\ge 0,\qquad
\sum_{t=1}^{L_s} c_{s,t}=1,\qquad
\sum_{t=1}^{L_s} r_{s,t}=R^{(s)}.
\end{equation}
Moreover, under the practical PPO instantiation
$A_t=A^{(\mathrm{seq})}[(1-\lambda)+\lambda\,\mathrm{clip}(w_t/\bar w)]$,
if $\mathrm{clip}(w_t/\bar w)\in[c_{\min},c_{\max}]$, then
\begin{equation}
|A_t|
\le
|A^{(\mathrm{seq})}|
\max\!\left\{
\big|(1-\lambda)+\lambda c_{\min}\big|,
\big|(1-\lambda)+\lambda c_{\max}\big|
\right\}.
\end{equation}
Therefore, Stage II yields a simplex-constrained and uniformly bounded
within-sentence refinement.
\end{lemma}

\begin{proof}
The simplex property follows directly from the definition
$c_{s,t}=\alpha_{s,t}/\sum_{k=1}^{L_s}\alpha_{s,k}$ and the positivity
assumption $\alpha_{s,t}>0$. In particular,
\begin{equation}
c_{s,t}\ge 0,
\qquad
\sum_{t=1}^{L_s} c_{s,t}
=
\frac{\sum_{t=1}^{L_s}\alpha_{s,t}}{\sum_{k=1}^{L_s}\alpha_{s,k}}
=1.
\end{equation}
The conservation property is then immediate from $r_{s,t}=R^{(s)}c_{s,t}$:
\begin{equation}
\sum_{t=1}^{L_s} r_{s,t}
=
R^{(s)}\sum_{t=1}^{L_s} c_{s,t}
=
R^{(s)}.
\end{equation}

For the practical PPO instantiation, write
\begin{equation}
m_t=(1-\lambda)+\lambda\,\mathrm{clip}(w_t/\bar w),
\qquad
A_t=A^{(\mathrm{seq})}m_t.
\end{equation}
Under the clipping assumption
$\mathrm{clip}(w_t/\bar w)\in[c_{\min},c_{\max}]$, we have
\begin{equation}
m_t \in \big[(1-\lambda)+\lambda c_{\min},\,
(1-\lambda)+\lambda c_{\max}\big].
\end{equation}
Therefore,
\begin{equation}
|A_t|
=
|A^{(\mathrm{seq})}|\cdot |m_t|
\le
|A^{(\mathrm{seq})}|
\max\!\left\{
\big|(1-\lambda)+\lambda c_{\min}\big|,
\big|(1-\lambda)+\lambda c_{\max}\big|
\right\}.
\end{equation}
This establishes the claimed boundedness.
\end{proof}

\begin{proposition}[Mitigation of global advantage alignment]
\label{prop:global_alignment}
Consider standard sequence-level RLHF with terminal reward propagation, where the token-level advantage is approximated by
\begin{equation}
A_t^{\mathrm{std}} \approx \gamma^{T-t} R(y) - V(s_t),
\end{equation}
with $\gamma \approx 1$. When the value baseline is imperfect, the shared sequence-level scalar $R(y)$ dominates $A_t^{\mathrm{std}}$ over a substantial portion of the trajectory, inducing sequence-wide alignment in the token-level updates.

Under S2T-RLHF, the sequence-level reward is first decomposed into sentence-level rewards $\{R^{(s)}\}_{s=1}^S$, and each sentence-level reward is then refined by simplex-constrained, uniformly bounded within-sentence modulation. Consequently, the resulting update signal is no longer driven by a single sequence-wide scalar, and the residual alignment that would remain under sentence-only uniform allocation is further attenuated. In this sense, S2T-RLHF weakens both cross-sentence alignment and residual within-sentence alignment relative to standard sequence-level RLHF.
\end{proposition}

\begin{proof}
Under standard sequence-level RLHF, the approximation
\begin{equation}
A_t^{\mathrm{std}} \approx \gamma^{T-t} R(y) - V(s_t)
\end{equation}
implies that, when $\gamma \approx 1$ and the value baseline is insufficiently accurate, a substantial fraction of the token-level advantages is dominated by the same global scalar $R(y)$. The resulting policy update therefore, takes the form
\begin{equation}
\Delta \theta \propto \sum_{t=1}^T A_t^{\mathrm{std}}
\nabla_\theta \log \pi_\theta(a_t\mid s_t),
\end{equation}
so that many tokens along the trajectory are driven by advantages with a similar sign and scale. This is the sequence-wide alignment effect
identified in Appendix~\ref{app:discussion}.

Under S2T-RLHF, Lemma~\ref{lem:stage1_localization} implies that the global reward $R(y)$ is replaced
by a conserved sentence-level decomposition
\begin{equation}
\sum_{s=1}^S R^{(s)} = R(y),
\end{equation}
So the optimization signal is no longer propagated only through a single sequence-wide scalar. Instead, different sentence groups inherit different coarse reward masses, which breaks the update homogeneity across sentences.

Lemma~\ref{lem:stage2_bounded_refinement} further shows that, within each sentence, Stage II applies a simplex-constrained and uniformly bounded refinement. Accordingly,
token-level modulation within a sentence is neither uniform nor unconstrained: informative local variation can be preserved without allowing arbitrary amplification. This directly mitigates the residual within-sentence alignment that would persist under sentence-only uniform allocation.

Taken together, the two-stage S2T structure weakens the sequence-wide alignment induced by standard terminal-reward propagation and further
reduces the residual alignment that remains at the sentence level.
\end{proof}

\begin{proposition}[Mitigation of variance amplification]
\label{prop:variance_amplification}
Consider standard sequence-level RLHF with token-level advantage
\begin{equation}
A_t^{\mathrm{std}} = \gamma^{T-t} R(y) - V(s_t),
\end{equation}
where $R(y)$ is a sequence-level terminal reward. When the covariance term $\mathrm{Cov}(R(y),V(s_t))$ remains limited, the variance of
$A_t^{\mathrm{std}}$ is dominated by the contribution of the global terminal reward, producing the long-horizon variance amplification
identified in Appendix~\ref{app:discussion}.

Under S2T-RLHF, the sequence-level reward is first redistributed over sentence-level semantic units and is then refined through bounded
within-sentence modulation. Consequently, exposure to a single full-sequence terminal reward is replaced by localized reward support together with uniformly bounded local refinement, which attenuates the variance amplification mechanism of standard sequence-level RLHF.
\end{proposition}

\begin{proof}
For standard sequence-level RLHF, Appendix~\ref{app:discussion} gives
\begin{equation}
\mathrm{Var}(A_t^{\mathrm{std}})
=
\mathrm{Var}(\gamma^{T-t}R(y)-V(s_t))
\ge
\gamma^{2(T-t)}\mathrm{Var}(R(y))
-2\gamma^{T-t}\mathrm{Cov}(R(y),V(s_t)).
\end{equation}
When the value baseline is inaccurate and the covariance term remains small, the dominant variance contribution comes from the shared terminal reward $R(y)$. In this regime, especially for early tokens, the advantage inherits variance exposure over the full effective horizon.

Under S2T-RLHF, Lemma~\ref{lem:stage1_localization} replaces the single sequence-level scalar by a conserved sentence-level decomposition,
\begin{equation}
R(y)=\sum_{s=1}^S R^{(s)},
\end{equation}
So the reward signal is no longer propagated only through a full-sequence terminal quantity. Instead, it is localized to sentence-conditioned token groups, thereby reducing the effective support over which variance is propagated.

Lemma~\ref{lem:stage2_bounded_refinement} further implies that the subsequent token-level refinement is simplex-constrained and uniformly bounded. In particular, the practical PPO modulation satisfies
\begin{equation}
|A_t|
\le
C_{\mathrm{clip}}\,|A^{(\mathrm{seq})}|,
\end{equation}
where $C_{\mathrm{clip}}<\infty$ depends only on $\lambda$ and the clipping range. Thus, the localized reward signal cannot be arbitrarily re-amplified at the token level.

Taken together, the two-stage S2T structure replaces full-horizon variance exposure under terminal-reward propagation by localized reward support with bounded local modulation, thereby mitigating the variance amplification mechanism identified in Appendix~\ref{app:discussion}.
\end{proof}

\begin{proposition}[Mitigation of length-dependent gradient scaling]
\label{prop:length_scaling}
Consider the policy gradient under standard sequence-level RLHF,
\begin{equation}
\nabla_\theta J(\theta)
\propto
\sum_{t=1}^T A_t^{\mathrm{std}}
\nabla_\theta \log \pi_\theta(a_t\mid s_t),
\end{equation}
where $A_t^{\mathrm{std}}$ is induced by sequence-level terminal reward propagation. When the token-level advantages are dominated by the same global reward signal, the cumulative gradient magnitude inherits an explicit dependence on the number of tokens receiving that signal,
yielding the length-dependent gradient scaling identified in Appendix~\ref{app:discussion}.

Under S2T-RLHF, the reward signal is first localized to sentence-level semantic units and is then redistributed within each sentence under
simplex and boundedness constraints. Consequently, the update magnitude is no longer governed solely by a single sequence-wide scalar propagated over all $T$ tokens, and the resulting dependence on total sequence length is attenuated relative to standard sequence-level RLHF.
\end{proposition}

\begin{proof}
Under standard sequence-level RLHF, when a substantial fraction of the token-level advantages is dominated by the same sequence-level reward
$R(y)$, these advantages tend to share a similar sign along the trajectory. The policy gradient therefore takes the form
\[
\nabla_\theta J(\theta)
\propto
\sum_{t=1}^T A_t^{\mathrm{std}}
\nabla_\theta \log \pi_\theta(a_t\mid s_t),
\]
so that the cumulative update magnitude grows with the number of tokens that inherit the common global signal. This induces the length-dependent coupling between optimization pressure and total sequence length analyzed
in Appendix~\ref{app:discussion}.

Under S2T-RLHF, Lemma~\ref{lem:stage1_localization} implies that the global reward is replaced by a conserved sentence-level decomposition. The update is therefore no longer driven exclusively by a single scalar propagated uniformly across the entire trajectory, but by a structured collection of sentence-local reward masses. This weakens the direct dependence of the update magnitude on the total number of generated tokens.

Lemma~\ref{lem:stage2_bounded_refinement} further shows that the subsequent token-level redistribution within each sentence is simplex-constrained and uniformly bounded.
Accordingly, localized refinement cannot assign arbitrarily large token weights that would recreate pathological scaling at the token level after sentence-level localization.

Taken together, the two-stage S2T structure replaces sequence-wide scalar propagation by localized and bounded reward redistribution, thereby attenuating the length-dependent gradient scaling mechanism identified in Appendix~\ref{app:discussion}.
\end{proof}

\begin{corollary}[End-to-end justification of the two-stage design] The two-stage structure of S2T-RLHF admits an end-to-end theoretical
justification relative to the instability mechanisms identified in Appendix~\ref{app:discussion}. In particular, Stage I induces sentence-level reward
localization, while Stage II provides bounded within-sentence refinement. Their combination jointly attenuates global advantage alignment, long-horizon variance amplification, and length-dependent gradient scaling.
\end{corollary}

\begin{proof}
Proposition~\ref{prop:global_alignment} shows that the two-stage structure weakens global advantage alignment by replacing sequence-wide reward propagation with sentence-localized and bounded token-level modulation. Proposition~\ref{prop:variance_amplification} shows that the same structure attenuates variance amplification by combining localized reward support with uniformly bounded local refinement. Proposition~\ref{prop:length_scaling} further shows that this combination weakens the dependence of gradient magnitude on total sequence length.

Taken together, these results identify complementary roles for the two stages. Stage I localizes the reward signal at the sentence level, thereby removing the sequence-wide mismatch induced by terminal-reward propagation. Stage II refines this localized signal within each sentence under explicit simplex and boundedness constraints, preventing the local redistribution from degenerating into unconstrained token-level amplification. The stated corollary follows.
\end{proof}

%% file: appendix/app_details.tex
\newpage
\section{Implementation Details}
\label{app:impl_details}
\subsection{Experiments Details}
\label{app:exper_details}
\subsubsection{Training Hyperparameters}

\begin{table}[!h]
  \caption{Training hyperparameters.}
  \label{tab:train_params}
  \begin{center}
    \begin{small}
      \begin{sc}
        \begin{tabular}{lc}
          \toprule
          \textbf{Parameter} & \textbf{Value} \\
          \midrule
            Prompt batch size         & 8 \\
            KL coefficient            & 0.15 \\
            Top-$p$                   & 0.9 \\
            Learning rate         & $3\times10^{-6}$ \\
            PPO clip range            & 0.1 \\
            PPO epochs                & 2 \\
            Batch size                & 16 \\
            Mini-batch size           & 2 \\
            Entropy coefficient       & 0.005 \\
            Value loss coefficient    & 0.2 \\
            GAE $\lambda$             & 0.95 \\
            \bottomrule
        \end{tabular}
      \end{sc}
    \end{small}
  \end{center}
  \vskip -0.1in
\end{table}

\subsubsection{Inter-Sentence Bargaining Solver}

\begin{table}[!h]
  \caption{Hyperparameters of the inter-sentence bargaining solver.}
  \label{tab:bargaining_config}
  \begin{center}
    \begin{small}
      \begin{sc}
        \begin{tabular}{lc}
          \toprule
          \textbf{Parameter} & \textbf{Value} \\
          \midrule
            Maximum iterations        & 100 \\
            Convergence tolerance     & $10^{-6}$ \\
            Damping factor            & 0.5 \\
            Numerical stability $\epsilon$ & $10^{-8}$ \\
          \bottomrule
        \end{tabular}
      \end{sc}
    \end{small}
  \end{center}
  \vskip -0.1in
\end{table}

\subsubsection{DTAN Architecture}

\begin{table}[!h]
  \caption{DTAN hyperparameters.}
  \label{tab:dtan_params}
  \begin{center}
    \begin{small}
      \begin{sc}
        \begin{tabular}{lc}
          \toprule
          \textbf{Parameter} & \textbf{Value} \\
          \midrule
            Input dimension  & 3584 \\
            Hidden dimension & 1024 \\
            Activation       & \texttt{GELU} \\
          \bottomrule
        \end{tabular}
      \end{sc}
    \end{small}
  \end{center}
  \vskip -0.1in
\end{table}

\subsubsection{LoRA Configuration}

\begin{table}[!h]
  \caption{LoRA hyperparameters used for policy fine-tuning.}
  \label{tab:lora_params}
  \begin{center}
    \begin{small}
      \begin{sc}
        \begin{tabular}{lc}
          \toprule
          \textbf{Parameter} & \textbf{Value} \\
          \midrule
            LoRA rank ($r$)        & 8 \\
            LoRA scaling ($\alpha$) & 16 \\
            LoRA dropout           & 0.05 \\
          \bottomrule
        \end{tabular}
      \end{sc}
    \end{small}
  \end{center}
  \vskip -0.1in
\end{table}

%% file: appendix/app_metrics_stable.tex
\newpage
\section{Quantitative Stability Metrics Across Seeds}
\label{app:stability_metrics_default}

To substantiate qualitative descriptors such as ``smoother'' and ``less oscillatory,''
we quantify training stability under the default configuration used in the main experiments
and aggregate results over three random seeds.

\paragraph{Logged trajectories.}
We record training dynamics at $T$ logging checkpoints using the same logging frequency and evaluation protocol for all methods and seeds.
At each checkpoint, we evaluate the current policy on a fixed set of prompts and log:
(i) KL divergence to the reference policy, (ii) the PPO policy objective,
(iii) dispersion of reward-model scores within the evaluation batch (reward variance), and
(iv) policy entropy.
All quantities are computed identically to the main training logs for every method and seed.
In particular, entropy is computed consistently by token-averaging over generated responses and then averaging over prompts.

\paragraph{Metrics.}
From the logged trajectories, we report four scalar stability summaries:
(i) \textbf{KL-Mean}, the time-average KL, measuring overall policy drift;
(ii) \textbf{Obj-$\Delta$}, the mean absolute step-to-step change of the PPO objective, quantifying objective oscillations;
(iii) \textbf{RVar-Mean}, the time-average reward variance within the evaluation batch, capturing dispersion of the preference signal;
and (iv) \textbf{Ent-Var}, the temporal variance of policy entropy, measuring volatility in exploration behavior.
Entropy is included mainly as a diagnostic to rule out degenerate collapse; its mean level need not be monotonic with stability,
whereas \textbf{Ent-Var} directly reflects fluctuations.
Lower values indicate more stable and better-controlled optimization dynamics.
%
%

\paragraph{Results.}
Tables~\ref{tab:stability_window_0_1000} and~\ref{tab:stability_window_5000_6000} summarize stability statistics over three seeds in two training phases.
Across both windows, sentence-structured credit assignment substantially reduces optimization oscillations: \textbf{Obj-$\Delta$} drops by an order of magnitude compared to vanilla RLHF, and S2T-RLHF achieves the smallest \textbf{Obj-$\Delta$} among all variants.
This supports our thesis that aligning reward decomposition with sentence-level semantics stabilizes PPO optimization, and that bounded within-sentence refinement further improves control.

For policy drift, token-only refinement exhibits the largest \textbf{KL-Mean} in the late window, consistent with amplified drift under overly fine-grained credit assignment.
S2T-RLHF reduces this drift relative to token-only refinement while maintaining the most stable objective dynamics.
Entropy volatility remains comparable across methods, with S2T-RLHF showing reduced \textbf{Ent-Var} in the early adaptation phase.

Finally, \textbf{RVar-Mean} increases for decomposition-based variants in the late phase.
Importantly, reward-score dispersion reflects evaluator variability rather than update volatility; with bounded refinement, S2T-RLHF can still yield smoother and better-controlled optimization even when reward scores are more dispersed.

\begin{table}[ht]
\centering
\small
\caption{Stability metrics in the early training window (0--1000 steps). Lower is better ($\downarrow$).}
\label{tab:stability_window_0_1000}
\begin{tabular}{lcccc}
\toprule
Method & RVar-Mean $\downarrow$ & KL-Mean $\downarrow$ & Obj-$\Delta$ $\downarrow$ & Ent-Var $\downarrow$ \\
\midrule
RLHF       & $0.1071 \pm 0.3397$ & $0.00405 \pm 0.00248$ & $0.01134 \pm 0.03368$ & $3.1156 \pm 1.7894$ \\
Sentence   & $0.2704 \pm 0.3914$ & $0.00372 \pm 0.00126$ & $0.00132 \pm 0.02201$ & $3.0894 \pm 1.8529$ \\
Token      & $0.2314 \pm 0.3832$ & $0.00455 \pm 0.00177$ & $0.00622 \pm 0.02341$ & $3.0754 \pm 1.8073$ \\
S2T-RLHF   & $0.2477 \pm 0.3657$ & $0.00418 \pm 0.00141$ & $0.00101 \pm 0.01299$ & $2.8496 \pm 1.7495$ \\
\bottomrule
\end{tabular}
\end{table}

\begin{table}[ht]
\centering
\small
\caption{Stability metrics in the late training window (5000--6000 steps). Lower is better ($\downarrow$).}
\label{tab:stability_window_5000_6000}
\begin{tabular}{lcccc}
\toprule
Method & RVar-Mean $\downarrow$ & KL-Mean $\downarrow$ & Obj-$\Delta$ $\downarrow$ & Ent-Var $\downarrow$ \\
\midrule
RLHF       & $0.2413 \pm 0.3319$ & $0.00523 \pm 0.00139$ & $0.01106 \pm 0.02562$ & $2.9479 \pm 1.7983$ \\
Sentence   & $0.4045 \pm 0.4640$ & $0.00619 \pm 0.00190$ & $0.00870 \pm 0.02125$ & $2.7376 \pm 1.7785$ \\
Token      & $0.4266 \pm 0.4147$ & $0.00902 \pm 0.00264$ & $0.00952 \pm 0.01699$ & $2.7834 \pm 1.7134$ \\
S2T-RLHF   & $0.4549 \pm 0.4239$ & $0.00837 \pm 0.00289$ & $0.00719 \pm 0.01361$ & $2.8489 \pm 1.9072$ \\
\bottomrule
\end{tabular}
\end{table}

\clearpage

%% file: appendix/app_stable.tex
\newpage
\section{Early-Stage Training Dynamics}
\label{app:early_stage}

This section presents a focused analysis of early-stage training dynamics in RLHF.
Empirically, we observe that training instability predominantly arises during the initial phase of optimization, where noise in credit assignment can be rapidly amplified before policy distributions stabilize. Accordingly, we provide a zoomed-in analysis of the first 500 optimization steps for different credit assignment strategies, complementing the full training trajectories reported in the main paper.

Beyond temporal localization, we further investigate whether the observed stability gains of S2T-RLHF depend on the learnable capacity of DTAN. In the standard S2T-RLHF setup, DTAN is initialized randomly and optimized jointly with the policy during RLHF training; we denote this variant as \textbf{S2T-noPT}. To assess the role of pretraining, we also consider a variant in which DTAN is first pretrained using the same data and reward model, and then held fixed during subsequent RLHF optimization; we denote this variant as \textbf{S2T-PT}.

Across settings, pretraining DTAN does not improve training stability and, in some cases, slightly increases early-stage reward variability. These results indicate that the stability benefits of S2T-RLHF arise from the structural regularization imposed by bounded, non-adaptive token-level refinement, rather than from increased representational capacity.

To improve interpretability of highly noisy early-stage trajectories, we visualize training curves using a smoothed estimate of the underlying dynamics. Specifically, for each metric we apply a centered sliding-window average over optimization steps, which acts as a low-pass filter to suppress high-frequency stochastic fluctuations induced by minibatch sampling and reward noise, while preserving long-term training trends. We additionally overlay the raw unsmoothed trajectories with reduced opacity to ensure faithful representation of the original data. This visualization strategy follows standard practice in reinforcement learning for analyzing optimization stability and convergence behavior, enabling clearer comparison of structural differences across credit assignment mechanisms.

\subsection{Early-stage dynamics under learning rate $5\times10^{-6}$}
\label{app:lr5e-6}

Figure~\ref{fig:lr5e-6} summarizes the first 500 PPO updates under a conservative learning rate ($5\times10^{-6}$), where stochasticity is moderate and behavioral differences primarily reflect the structure of credit assignment. We compare standard RLHF, sentence-only and token-only decompositions, and S2T-RLHF with/without DTAN pretraining.

\paragraph{Entropy.}
All methods keep policy entropy within a narrow band, indicating no early mode collapse in this regime. However, token-level decomposition exhibits visibly larger entropy fluctuations, consistent with higher sensitivity to local reward noise. Sentence-level decomposition and both S2T variants yield smoother entropy trajectories, suggesting more stable early exploration.

\paragraph{KL to the reference policy.}
Token-level decomposition shows a clearer upward drift in KL even at this small step size, implying more aggressive deviation from the reference distribution. In contrast, standard RLHF and sentence-level decomposition maintain tighter KL control, and S2T-RLHF matches or improves upon this behavior, indicating that structured sentence-to-token allocation can regulate early policy drift.

\paragraph{Objective, reward variance, and critic learning.}
Token-level decomposition attains higher short-term objective values, but with substantial variability, consistent with noisy (high-variance) advantage estimates. S2T-RLHF exhibits a smoother and more monotonic objective evolution and lower reward variance, indicating reduced amplification of stochastic learning signals. Value-function loss decreases across all methods, but S2T-RLHF shows the most consistent early convergence. DTAN pretraining does not improve stability and can slightly increase variability.

\paragraph{Takeaway.}
Under $5\times10^{-6}$, S2T-RLHF achieves stable early optimization while avoiding the larger variance and stronger drift induced by fine-grained token-level credit assignment. The limited benefit of DTAN pretraining supports the view that stability gains mainly arise from structural regularization rather than added capacity.

\subsection{Early-stage dynamics under learning rate $8\times10^{-6}$}
\label{app:lr8e-6}

Figure~\ref{fig:lr8e-6} reports the first 500 updates under a moderately larger learning rate ($8\times10^{-6}$), which acts as a stronger stress test where optimization noise is more easily amplified.

\paragraph{Entropy.}
Entropy remains bounded across methods but becomes visibly more volatile than at $5\times10^{-6}$. Token-level decomposition shows the largest fluctuations, while S2T-RLHF maintains comparatively smooth entropy evolution, indicating more robust exploration behavior as step sizes increase.

\paragraph{KL to the reference policy.}
Differences are most pronounced in KL. Token-level decomposition exhibits rapid and sustained KL growth, reflecting accelerated policy drift under higher learning rates. Standard RLHF and sentence-level decomposition retain tighter control, and S2T-RLHF achieves consistently lower KL growth throughout early training. The pretrained DTAN variant tends to drift slightly more than the non-pretrained one.

\paragraph{Objective, reward variance, and critic learning.}
Token-level decomposition can reach higher objective values later in this window, but these gains coincide with large oscillations and higher reward variance. In contrast, S2T-RLHF improves the objective more smoothly and maintains more consistent reward statistics. Value-function loss decreases for all methods, but S2T-RLHF shows a more monotonic and stable convergence pattern.

\paragraph{Takeaway.}
At $8\times10^{-6}$, instability from token-level credit assignment becomes more apparent, primarily through variance amplification and excessive drift. S2T-RLHF remains stable and better controlled, again suggesting that structured allocation, not additional capacity, drives the robustness gains.

\subsection{Early-stage dynamics under learning rate $1\times10^{-5}$}
\label{app:lr1e-5}

Figure~\ref{fig:lr1e-5} reports early-stage dynamics under an aggressive learning rate ($1\times10^{-5}$), near the boundary where RLHF training often becomes unstable.

\paragraph{Entropy.}
Entropy stays bounded on average but exhibits markedly larger fluctuations than in lower learning-rate settings. Token-level decomposition remains the most volatile, indicating unstable exploration under aggressive updates. Sentence-level decomposition partially moderates this effect, while S2T-RLHF shows smoother entropy evolution.

\paragraph{KL to the reference policy.}
Token-level decomposition displays the fastest and most persistent KL growth, consistent with excessive policy drift. S2T-RLHF substantially mitigates this drift relative to token-only methods, while standard RLHF and sentence-level decomposition typically retain the tightest KL control. As in other regimes, DTAN pretraining does not reduce drift and often slightly worsens it.

\paragraph{Objective, reward variance, and critic learning.}
Token-level decomposition can produce larger objective values, but with frequent oscillations and elevated reward variance, consistent with high-variance advantages under large step sizes. S2T-RLHF yields a more gradual and stable objective improvement, alongside lower reward variance and smoother critic convergence.

\paragraph{Takeaway.}
Under $1\times10^{-5}$, fine-grained token-level credit assignment leads to pronounced instability, which is most clearly reflected in increased policy drift and amplified reward variability, whereas S2T-RLHF remains comparatively robust, supporting the need for structured and bounded credit assignment in aggressive optimization regimes.

\subsection{Early-stage dynamics under reduced KL constraint}
\label{app:kl01}

Figure~\ref{fig:kl01} repeats the early-stage analysis under a stricter KL constraint ($0.1$), which directly limits deviation from the reference policy and complements the learning-rate-based stress tests.

As expected, the tighter KL constraint bounds policy drift for all methods, compressing differences in KL trajectories relative to looser settings. Nevertheless, stability differences remain.

\paragraph{Entropy and objective.}
Entropy variance is reduced overall, but token-level decomposition continues to exhibit larger fluctuations than sentence-level and S2T-RLHF variants, indicating persistent sensitivity to stochastic reward signals even under stronger regularization. With tighter KL control, objective gains are also more conservative: token-level decomposition still shows slightly higher short-term objectives, but with noticeable variability, whereas S2T-RLHF remains smoother and more consistent.

\paragraph{Reward variance and critic learning.}
Reward variance decreases across methods under stronger KL regularization, yet token-level decomposition still exhibits relatively higher fluctuations. S2T-RLHF maintains lower and more stable reward variance and shows smoother value-loss convergence, suggesting more reliable critic learning under bounded token-level refinement.

\paragraph{Takeaway.}
A stricter KL constraint stabilizes RLHF for all methods but does not eliminate structural differences across credit assignment strategies. S2T-RLHF remains more stable than token-level decomposition even when policy updates are tightly constrained, implying that the gains cannot be explained solely by KL regularization and instead arise from structured reward allocation.

\begin{figure}[!ht]
  \vskip 0.2in
  \centering

  \begin{subfigure}{0.49\columnwidth}
    \centering
    \includegraphics[width=\linewidth]{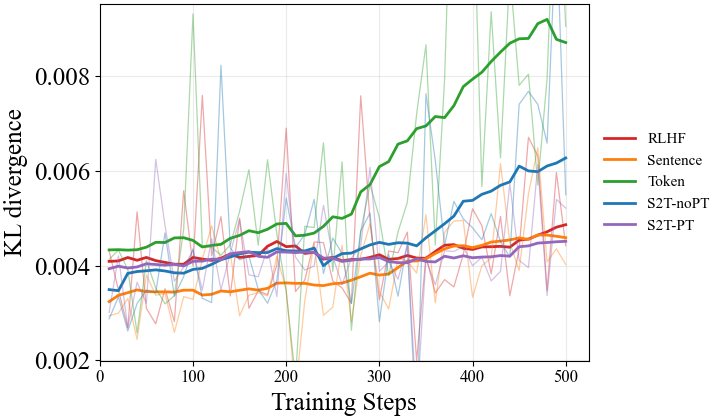}
    \caption{KL Divergence}
    \label{fig:5_kl}
  \end{subfigure}
  \hfill
  \begin{subfigure}{0.49\columnwidth}
    \centering
    \includegraphics[width=\linewidth]{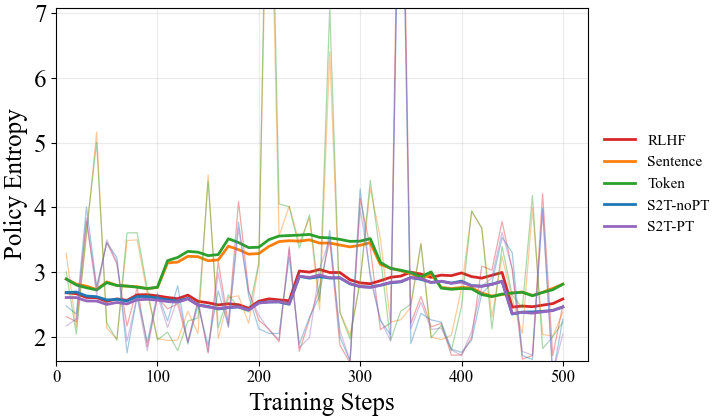}
    \caption{Policy Entropy}
    \label{fig:5_entropy}
  \end{subfigure}

  \vskip 0.15in

  \begin{subfigure}{0.49\columnwidth}
    \centering
    \includegraphics[width=\linewidth]{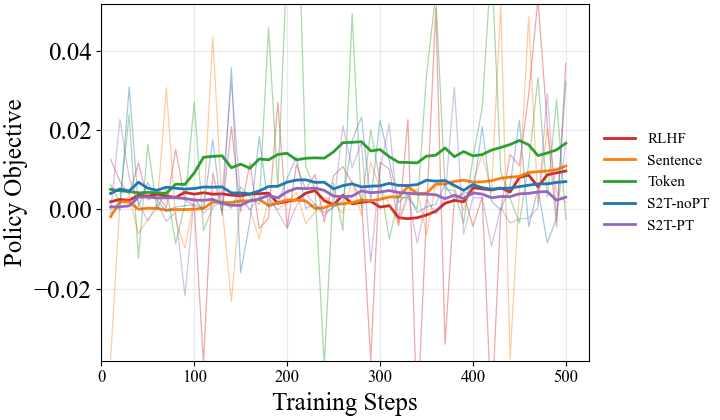}
    \caption{Policy Objective}
    \label{fig:5_policy}
  \end{subfigure}
  \hfill
  \begin{subfigure}{0.49\columnwidth}
    \centering
    \includegraphics[width=\linewidth]{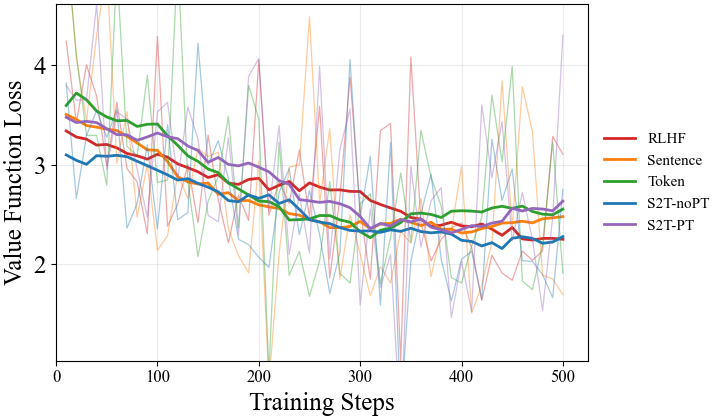}
    \caption{Value Function Loss}
    \label{fig:5_value}
  \end{subfigure}

  \vskip 0.15in

  \begin{subfigure}{0.49\columnwidth}
    \centering
    \includegraphics[width=\linewidth]{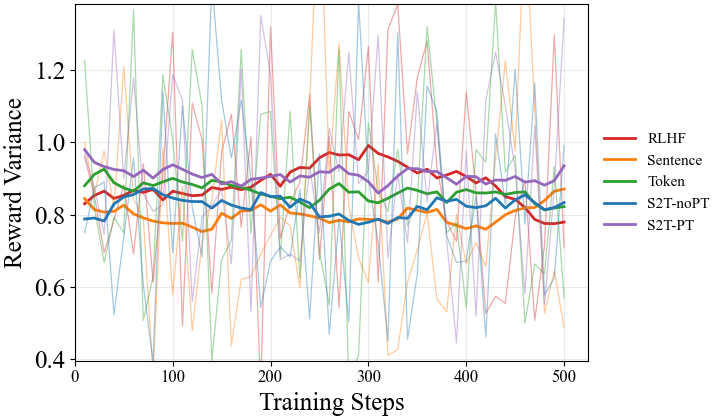}
    \caption{Reward Variance}
    \label{fig:5_reward}
  \end{subfigure}

  \caption{Early-stage training dynamics under learning rate $5\times10^{-6}$.}
  \label{fig:lr5e-6}
\end{figure}


\begin{figure}[!ht]
  \vskip 0.2in
  \centering

  \begin{subfigure}{0.49\columnwidth}
    \centering
    \includegraphics[width=\linewidth]{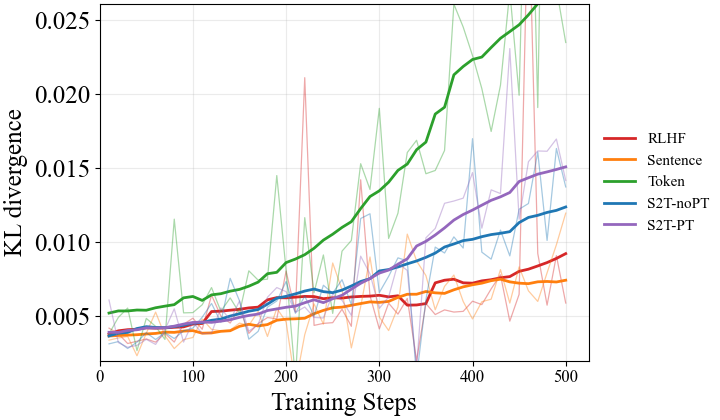}
    \caption{KL Divergence}
    \label{fig:8_kl}
  \end{subfigure}
  \hfill
  \begin{subfigure}{0.49\columnwidth}
    \centering
    \includegraphics[width=\linewidth]{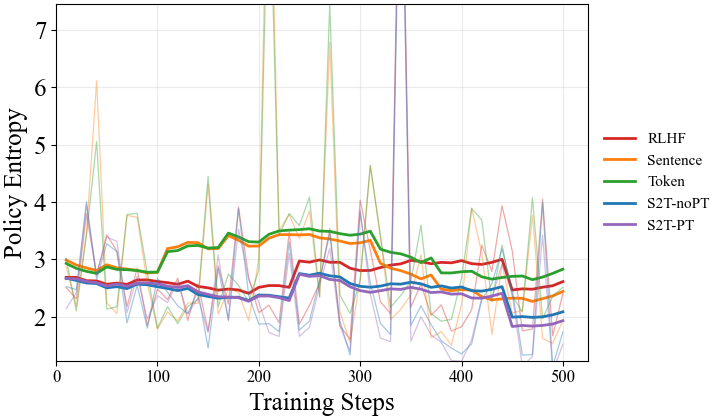}
    \caption{Policy Entropy}
    \label{fig:8_entropy}
  \end{subfigure}

  \vskip 0.15in

  \begin{subfigure}{0.49\columnwidth}
    \centering
    \includegraphics[width=\linewidth]{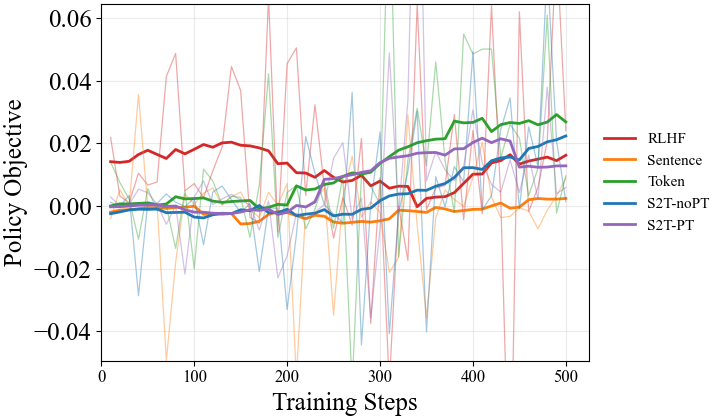}
    \caption{Policy Objective}
    \label{fig:8_policy}
  \end{subfigure}
  \hfill
  \begin{subfigure}{0.49\columnwidth}
    \centering
    \includegraphics[width=\linewidth]{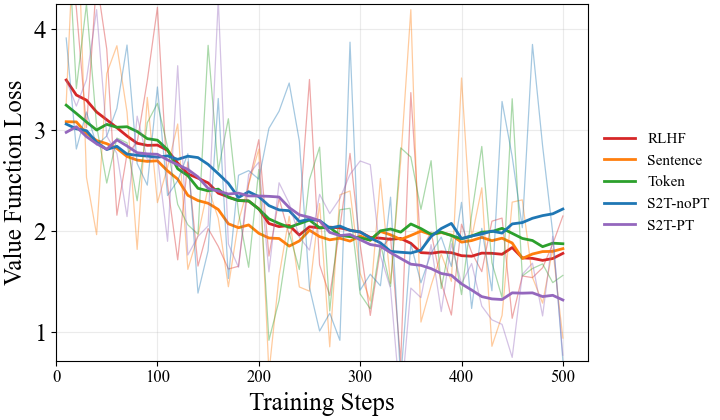}
    \caption{Value Function Loss}
    \label{fig:8_value}
  \end{subfigure}

  \vskip 0.15in

  \begin{subfigure}{0.49\columnwidth}
    \centering
    \includegraphics[width=\linewidth]{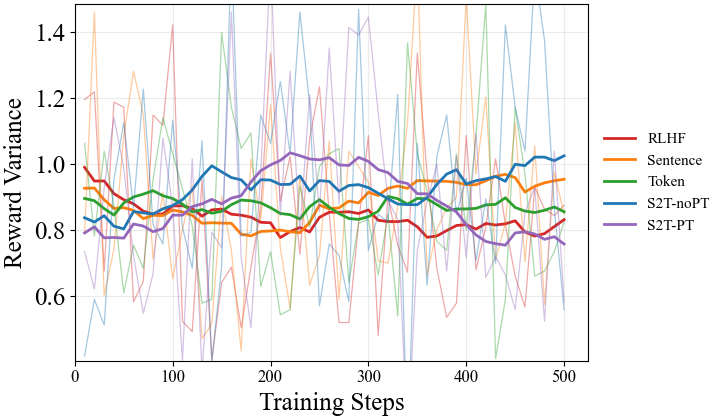}
    \caption{Reward Variance}
    \label{fig:8_reward}
  \end{subfigure}

  \caption{Early-stage training dynamics under learning rate $8\times10^{-6}$.}
  \label{fig:lr8e-6}
\end{figure}


\begin{figure}[!ht]
  \vskip 0.2in
  \centering

  \begin{subfigure}{0.49\columnwidth}
    \centering
    \includegraphics[width=\linewidth]{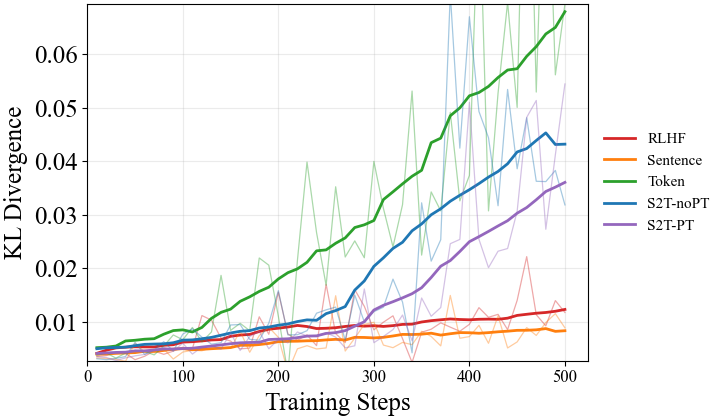}
    \caption{KL Divergence}
    \label{fig:1_kl}
  \end{subfigure}
  \hfill
  \begin{subfigure}{0.49\columnwidth}
    \centering
    \includegraphics[width=\linewidth]{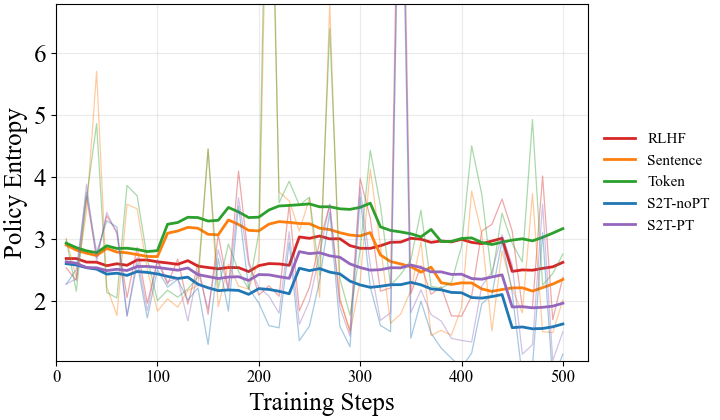}
    \caption{Policy Entropy}
    \label{fig:1_entropy}
  \end{subfigure}

  \vskip 0.15in

  \begin{subfigure}{0.49\columnwidth}
    \centering
    \includegraphics[width=\linewidth]{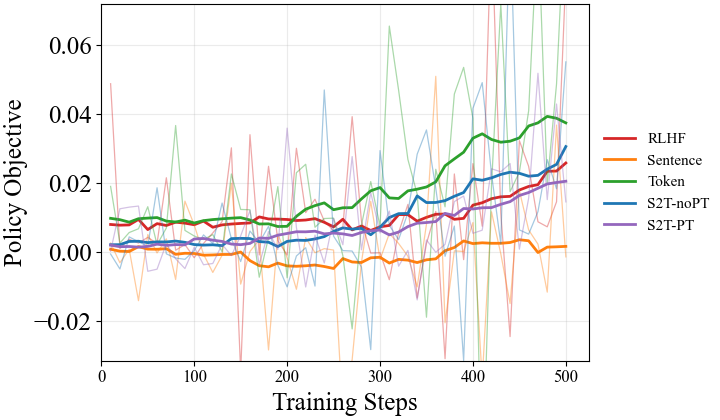}
    \caption{Policy Objective}
    \label{fig:1_policy}
  \end{subfigure}
  \hfill
  \begin{subfigure}{0.49\columnwidth}
    \centering
    \includegraphics[width=\linewidth]{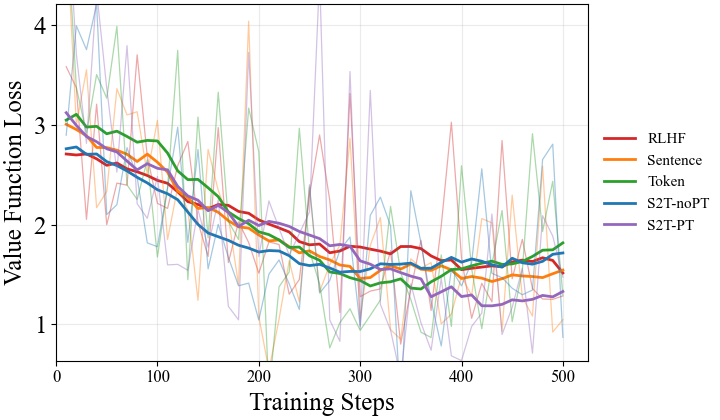}
    \caption{Value Function Loss}
    \label{fig:1_value}
  \end{subfigure}

  \vskip 0.15in

  \begin{subfigure}{0.49\columnwidth}
    \centering
    \includegraphics[width=\linewidth]{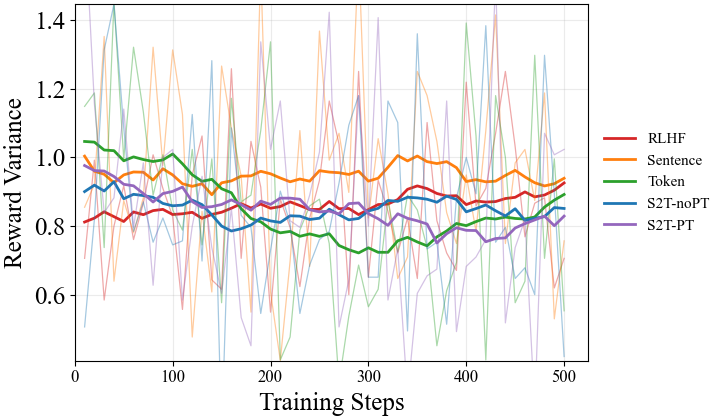}
    \caption{Reward Variance}
    \label{fig:1_reward}
  \end{subfigure}

  \caption{Early-stage training dynamics under learning rate $1\times10^{-5}$.}
  \label{fig:lr1e-5}
\end{figure}


\begin{figure}[!ht]
  \vskip 0.2in
  \centering

  \begin{subfigure}{0.49\columnwidth}
    \centering
    \includegraphics[width=\linewidth]{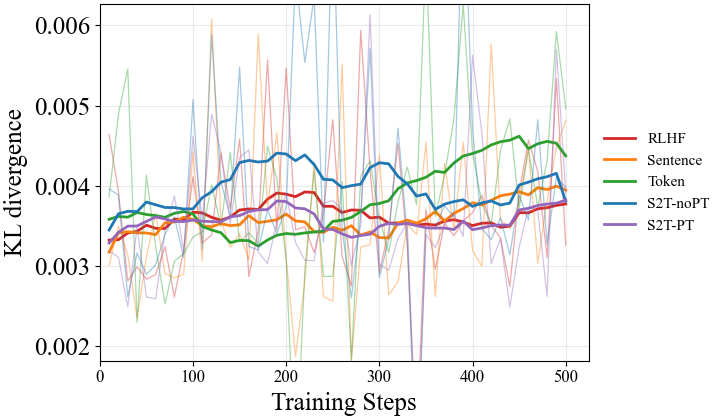}
    \caption{KL Divergence}
    \label{fig:kl01_kl}
  \end{subfigure}
  \hfill
  \begin{subfigure}{0.49\columnwidth}
    \centering
    \includegraphics[width=\linewidth]{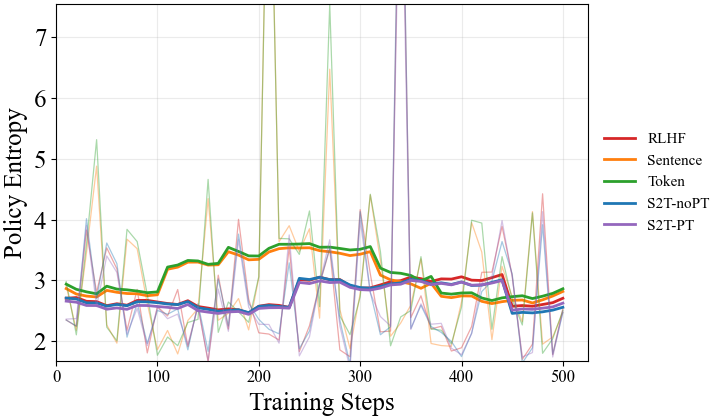}
    \caption{Policy Entropy}
    \label{fig:kl01_entropy}
  \end{subfigure}

  \vskip 0.15in

  \begin{subfigure}{0.49\columnwidth}
    \centering
    \includegraphics[width=\linewidth]{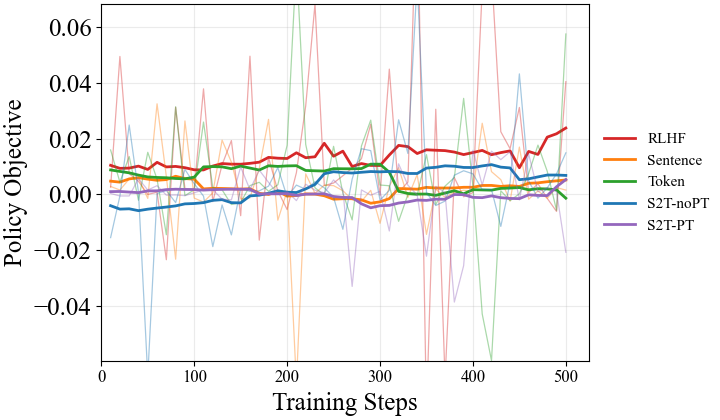}
    \caption{Policy Objective}
    \label{fig:kl01_policy}
  \end{subfigure}
  \hfill
  \begin{subfigure}{0.49\columnwidth}
    \centering
    \includegraphics[width=\linewidth]{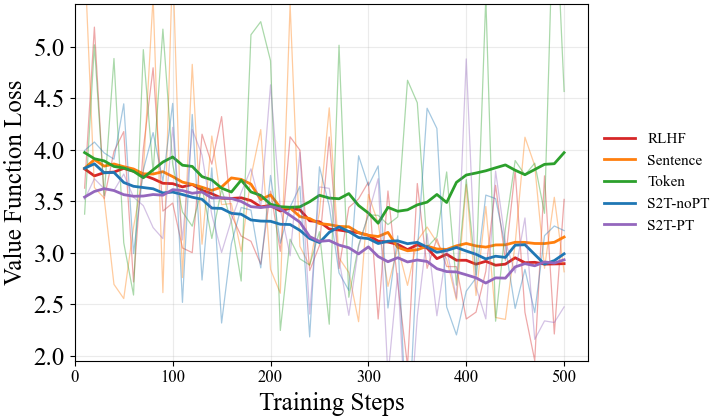}
    \caption{Value Function Loss}
    \label{fig:kl01_value}
  \end{subfigure}

  \vskip 0.15in

  \begin{subfigure}{0.49\columnwidth}
    \centering
    \includegraphics[width=\linewidth]{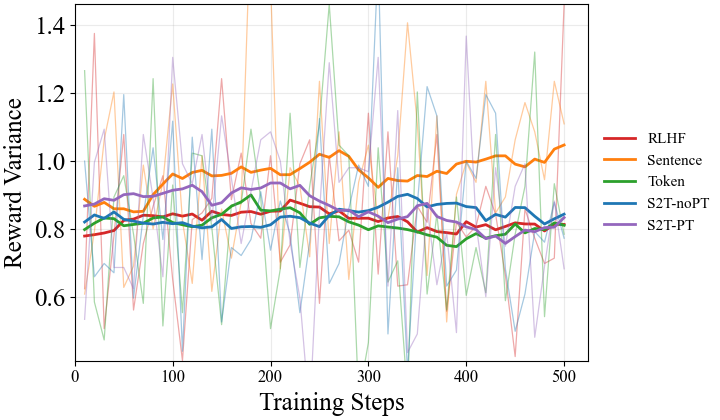}
    \caption{Reward Variance}
    \label{fig:kl01_reward}
  \end{subfigure}

  \caption{Early-stage training dynamics under KL constraint $0.1$.}
  \label{fig:kl01}
\end{figure}

\clearpage

%% file: appendix/app_ex_tab.tex
%
%

\newpage
\section{Supplementary Experimental Results}
\label{app:supplement_ex}

\subsection{Cross-Dataset Preference Evaluation}
\label{app:cross}

\paragraph{Datasets.}
We evaluate the robustness and generalization of S2T-RLHF across a diverse collection of alignment and safety benchmarks. As an in-distribution reference, we use \textbf{HH-RLHF}~\cite{bai2022training}, which consists of human preference comparisons targeting helpfulness and harmlessness. To assess out-of-distribution generalization under adversarial settings, we include \textbf{AdvBench}~\cite{zou2023universal}, a benchmark designed to probe jailbreak-style and instruction-following vulnerabilities. We further consider toxicity-oriented datasets, including \textbf{RealToxicityPrompts}~\cite{gehman2020realtoxicityprompts} and \textbf{ToxiGen}~\cite{hartvigsen2022toxigen}, which focus on harmful, biased, and offensive language generation. In addition, we evaluate on \textbf{SafeRLHF}, a curated safety preference dataset covering refusal behavior, harmlessness, and policy compliance. Together, these datasets span a wide range of prompt distributions, risk profiles, and alignment objectives, enabling a comprehensive evaluation of cross-dataset generalization beyond the training domain. Unless otherwise specified, models are trained on HH-RLHF and evaluated zero-shot on the remaining datasets.

\paragraph{Reward Models.}
We decouple the reward models used for training and evaluation to assess robustness across reward-model choices. During RLHF optimization, all methods are trained using \textbf{RM-Gemma-2B}~\cite{dong2023raft}, which provides the scalar preference signal for policy updates under a standard PPO-based RLHF pipeline. For evaluation, we adopt a stronger and independently trained reward model, \textbf{RM-Gemma-7B}~\cite{dong2023raft}, to measure preference alignment and generalization. RM-Gemma-7B follows a standard instruction-tuned reward modeling pipeline and is widely used in recent RLHF studies. Unless otherwise specified, all reported preference alignment results are evaluated using RM-Gemma-7B.

\begin{table}[!h]
\centering
\small
\caption{Pairwise preference win rates against RLHF under RM-Gemma-7B across datasets.}
\label{tab:winrate-rm7b}
\begin{tabular}{lcccccc}
\toprule
Dataset 
& \multicolumn{2}{c}{vs. RLHF} 
& \multicolumn{2}{c}{vs. ABC} 
& \multicolumn{2}{c}{vs. SCAR} \\
\cmidrule(lr){2-3}\cmidrule(lr){4-5}\cmidrule(lr){6-7}
& Win$\uparrow$ & Lose 
& Win$\uparrow$ & Lose 
& Win$\uparrow$ & Lose \\
\midrule

HH-RLHF   & \textbf{51.1} & 39.5 & \textbf{48.7} & 43.8 & \textbf{45.8} & 44.1 \\
Toxicity  & \textbf{48.4} & 45.4 & 46.9 & \textbf{48.9} & \textbf{51.6} & 44.4 \\
AdvBench  & \textbf{52.5} & 43.4 & \textbf{52.5} & 41.4 & 44.4 & \textbf{49.6} \\
ToxiGen   & \textbf{46.6} & 43.4 & 42.4 & \textbf{46.4} & \textbf{47.7} & 47.5 \\
SafeRLHF  & \textbf{53.5} & 45.3 & \textbf{48.7} & 48.1 & \textbf{50.1} & 45.9 \\

\bottomrule
\end{tabular}
\vskip -0.1in
\end{table}

Table~\ref{tab:winrate-rm7b} reports pairwise preference outcomes evaluated under RM-Gemma-7B across multiple datasets. Overall, S2T-RLHF attains a win advantage over RLHF on all listed datasets, suggesting that introducing stability-oriented decomposition does not harm preference alignment under the same evaluator and can yield modest improvements. Against stronger baselines, namely ABC and SCAR, the results are more dataset-dependent. S2T-RLHF wins on several datasets, for example Toxicity and SafeRLHF against SCAR, while trailing in a small number of pairings, for example Toxicity against ABC and AdvBench against SCAR. Taken together, these cross-dataset outcomes indicate that S2T-RLHF is generally effective for preference alignment, and the remaining failures are concentrated in a few difficult dataset and baseline combinations rather than reflecting a systematic degradation.

\subsection{Model and Reward-Model Robustness}
\label{app:model_rm_robustness}

To assess whether the preference-alignment conclusions depend on a specific policy backbone or training reward model, we conduct additional experiments with two model--reward-model configurations. In the Gemma-based setting, we use Gemma-2-2B as the policy model and RM-Gemma-7B as the training reward model. In the Qwen-based setting, we use Qwen3-4B as the policy model and Skywork-Reward-V2-Qwen3-4B as the training reward model. All other training and evaluation settings, including prompts, compared model pairs, decoding configuration, and win--tie--lose computation, are kept the same as in the main preference evaluation. Preference alignment is evaluated using the same GPT-4o-mini LLM-as-a-Judge protocol as in main experiments.

Table~\ref{tab:model_rm_robustness} reports the pairwise preference results. Across both model--reward-model configurations, S2T-RLHF maintains a clear win advantage over standard RLHF and ABC, and also achieves higher win rates than loss rates against SCAR. These results suggest that the preference-alignment conclusion is not tied to a single policy backbone or training reward model.

\begin{table}[h]
\centering
\caption{Model and reward-model robustness results evaluated using GPT-4o-mini as the LLM-as-a-Judge.}
\label{tab:model_rm_robustness}
\begin{tabular}{llccc}
\toprule
Training setup & Method & Win (\%) $\uparrow$ & Lose (\%) $\downarrow$ & Tie (\%) \\
\midrule
Gemma-2-2B / RM-Gemma-7B & vs. RLHF & 58.1 & 28.4 & 13.5 \\
Gemma-2-2B / RM-Gemma-7B & vs. ABC  & 50.9 & 31.8 & 17.3 \\
Gemma-2-2B / RM-Gemma-7B & vs. SCAR & 45.6 & 35.3 & 19.1 \\
\midrule
Qwen3-4B / Skywork-Reward-V2-Qwen3-4B & vs. RLHF & 56.8 & 29.8 & 13.4 \\
Qwen3-4B / Skywork-Reward-V2-Qwen3-4B & vs. ABC  & 50.9 & 31.5 & 17.6 \\
Qwen3-4B / Skywork-Reward-V2-Qwen3-4B & vs. SCAR & 46.3 & 35.6 & 18.1 \\
\bottomrule
\end{tabular}
\end{table}

Table~\ref{tab:model_rm_robustness} reports the pairwise preference results. Across both model--reward-model configurations, S2T-RLHF maintains a clear win advantage over standard RLHF and ABC, and also achieves higher win rates than loss rates against SCAR. In the Gemma-based setting, S2T-RLHF improves over RLHF by a large margin, with 58.1\% wins versus 28.4\% losses, and remains favorable against ABC and SCAR. The Qwen-based setting shows a similar pattern, with 56.8\% wins against RLHF and consistent win advantages over both ABC and SCAR. The margins against SCAR are smaller than those against RLHF and ABC, which is expected since SCAR is a strong fine-grained credit-assignment baseline. Overall, these results suggest that the preference-alignment conclusion is not tied to a single policy backbone or training reward model.

\subsection{Mechanism-Swap Ablations}
\label{app:mechanism_swap}

The stage-level ablations in the main text examine whether sentence-level allocation, token-level refinement, and their combination are useful. However, they do not isolate whether the gains come only from the hierarchical sentence-to-token structure or also from the specific allocation mechanisms used in each stage. To address this question, we conduct mechanism-swap ablations while keeping the overall two-stage structure fixed.

For Stage I, we replace the Nash-based sentence allocator with two simpler alternatives: uniform sentence allocation and length-proportional allocation. Uniform allocation assigns the same reward share to each sentence, while length-proportional allocation distributes reward according to sentence length. For Stage II, we replace DTAN with uniform within-sentence allocation and a linear+softmax token allocator. These variants preserve the same sentence-to-token decomposition pipeline, but change the mechanism used to allocate reward mass within each stage.

Table~\ref{tab:stage1_mechanism_swap} reports the Stage-I mechanism-swap results. The Nash allocator achieves the best overall trade-off, with the lowest KL-Mean and Obj-$\Delta$, as well as the highest win rate against RLHF. Compared with uniform and length-proportional allocation, Nash allocation provides smoother optimization dynamics while preserving stronger preference alignment.

\begin{table}[h]
\centering
\caption{Stage-I mechanism-swap ablations. Lower KL-Mean and Obj-$\Delta$ indicate more stable optimization. Preference results are evaluated against RLHF.}
\label{tab:stage1_mechanism_swap}
\begin{tabular}{lccccc}
\toprule
Stage-I allocator & KL-Mean $\downarrow$ & Obj-$\Delta$ $\downarrow$ & Win (\%) $\uparrow$ & Tie (\%) & Lose (\%) $\downarrow$ \\
\midrule
Uniform      & 0.00455 & 0.00622 & 51.4 & 9.8 & 38.8 \\
Length-prop. & 0.00443 & 0.00229 & 51.1 & 7.6 & 41.3 \\
Nash (ours)  & 0.00418 & 0.00101 & 53.3 & 8.4 & 38.3 \\
\bottomrule
\end{tabular}
\end{table}

Table~\ref{tab:stage2_mechanism_swap} reports the Stage-II mechanism-swap results. DTAN achieves the strongest overall configuration, with the lowest KL-Mean and Obj-$\Delta$, and the highest win rate against RLHF. Uniform within-sentence allocation is stable but less effective in preference alignment, while the linear+softmax allocator improves win rate but yields larger objective oscillations than DTAN. This suggests that normalized token-level refinement is helpful, but the specific DTAN parameterization provides a better stability--alignment trade-off.

\begin{table}[h]
\centering
\caption{Stage-II mechanism-swap ablations. Lower KL-Mean and Obj-$\Delta$ indicate more stable optimization. Preference results are evaluated against RLHF.}
\label{tab:stage2_mechanism_swap}
\begin{tabular}{lccccc}
\toprule
Stage-II allocator & KL-Mean $\downarrow$ & Obj-$\Delta$ $\downarrow$ & Win (\%) $\uparrow$ & Tie (\%) & Lose (\%) $\downarrow$ \\
\midrule
Uniform          & 0.00572 & 0.00132 & 48.5 & 12.4 & 39.1 \\
Linear+softmax   & 0.00452 & 0.00192 & 52.9 & 8.6  & 38.5 \\
DTAN (ours)      & 0.00418 & 0.00101 & 53.3 & 8.4  & 38.3 \\
\bottomrule
\end{tabular}
\end{table}

Overall, these mechanism-swap ablations show that the gains of S2T-RLHF are not explained by hierarchical decomposition alone. The Nash-based Stage-I allocator improves sentence-level reward localization, while DTAN provides bounded and adaptive within-sentence refinement. Their combination yields the best stability--alignment trade-off among the tested configurations.

\subsection{Training-time Reward Monitoring.}
\label{app:train_time_re}
Figure~\ref{fig:train_reward} visualizes the reward trajectories recorded during training, as measured by the optimization reward model (RM-Gemma-2B). This diagnostic is included to illustrate how different credit assignment strategies interact with the training reward signal over time, rather than to serve as a performance metric. We observe that token-level methods (e.g., ABC and SCAR) tend to exhibit larger short-term reward fluctuations, suggesting more aggressive and localized policy updates. In contrast, S2T-RLHF maintains comparatively smoother reward trajectories throughout training, indicating more structured and controlled utilization of the reward signal. Importantly, these differences in training-time reward dynamics do not necessarily translate into large gaps in final preference alignment scores (Table~\ref{tab:winrate-rm7b}), highlighting the limited interpretability of raw reward values during RLHF optimization and motivating the need for stability-aware evaluation.

\begin{figure}[th]
  \centering
  \includegraphics[width=0.65\columnwidth]{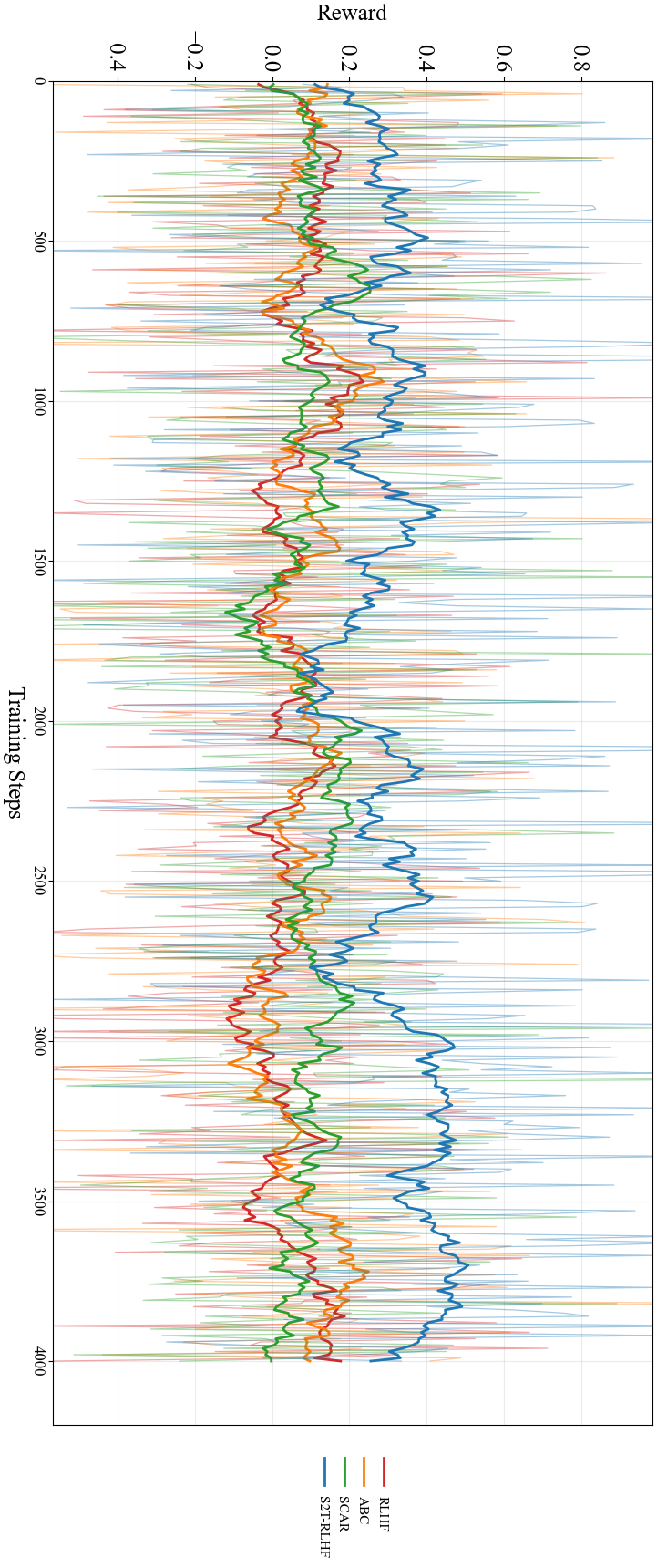}
  \vspace{-0.25cm}
  \caption{Training-time reward trajectories under the optimization reward model (RM-Gemma-2B).}
  \label{fig:train_reward}
\end{figure}

\clearpage

%% file: appendix/app_case.tex
\newpage
\section{A Case Study on Noise Amplification under Fine-Grained Reward Decomposition.}
\label{app:case}

Figure~\ref{fig:case} presents an illustrative training run under a more aggressive learning rate of $1\times10^{-5}$. While all methods exhibit increasing KL divergence, indicating faster policy drift from the reference model, their reward dynamics differ markedly. For ABC and SCAR, the KL increase is accompanied by pronounced oscillations and a sustained degradation of the training reward that crosses below zero and remains negative for a substantial portion of training. This pattern, characterized by continued policy drift together with worsening reward, is consistent with noise amplification under highly localized credit signals rather than improved preference optimization. Moreover, ABC exhibits a clear regime change in the later stage, around $\sim$450 steps, where the KL curve becomes highly volatile with large spikes, indicating unstable and noise-driven updates.

In contrast, both variants of S2T-RLHF maintain smoother and more monotonic KL growth under the same learning rate, without entering the negative-reward regime observed for ABC and SCAR. Although the policy still departs from the reference, the reward signal remains comparatively stable throughout training, suggesting that sentence-level allocation with bounded token-level refinement mitigates noise amplification during aggressive optimization. Overall, this case study supports the view that overly fine-grained reward decomposition can destabilize RLHF training by concentrating noisy preference signals, especially under larger learning rates.

\begin{figure}[ht]
  \vskip 0.2in
  \centering

  \begin{subfigure}{0.49\columnwidth}
    \centering
    \includegraphics[width=\linewidth]{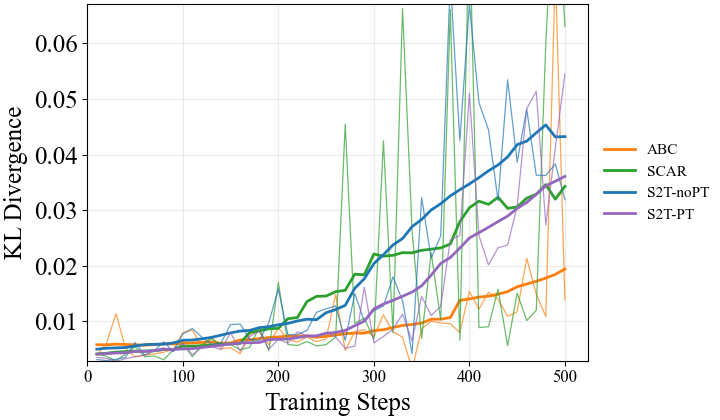}
    \caption{KL Divergence}
    \label{fig:case_1_kl}
  \end{subfigure}
  \hfill
  \begin{subfigure}{0.49\columnwidth}
    \centering
    \includegraphics[width=\linewidth]{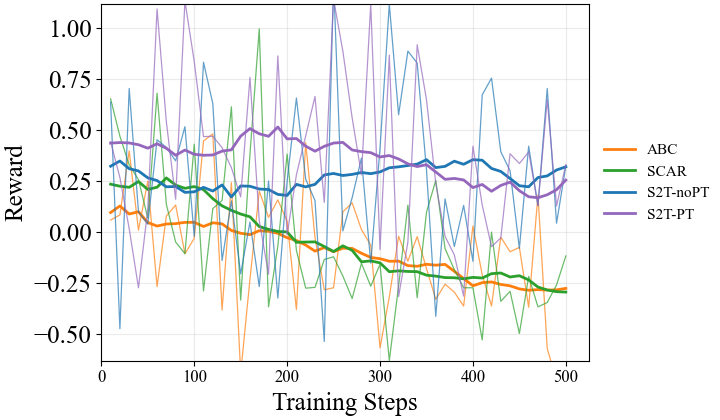}
    \caption{Reward Mean}
    \label{fig:case_1_r}
  \end{subfigure}

  \caption{Illustrative training dynamics under fine-grained reward decomposition with a larger learning rate ($1\times10^{-5}$).}
  \label{fig:case}
\end{figure}

%% file: appendix/app_compute.tex
\newpage
\section{Computational Resources and Runtime Analysis}
\label{app:compute}

\paragraph{Hardware and training budget.}
All experiments were conducted on a single A100-SXM4-80GB GPU. Unless otherwise specified, all methods were trained under the same PPO optimization budget, prompt batch size, mini-batch size, KL coefficient, LoRA configuration, and evaluation protocol. We use Gemma-2-9B as the policy model and RM-Gemma-2B as the reward model.

\paragraph{Per-step runtime comparison.}
Table~\ref{tab:runtime_comparison} reports the average per-step runtime of different credit assignment methods. S2T-RLHF incurs additional computation compared with standard RLHF and ABC, but remains more efficient than SCAR. Specifically, S2T-RLHF requires 21.18 seconds per PPO step, corresponding to an 81.5\% overhead over RLHF, whereas SCAR requires 26.98 seconds per step, corresponding to a 131.2\% overhead. This result shows that S2T-RLHF introduces nontrivial but moderate computational overhead relative to Shapley-based token decomposition.

\begin{table}[h]
\centering
\caption{Per-step runtime comparison. All methods are evaluated under the same hardware and PPO training budget.}
\label{tab:runtime_comparison}
\begin{tabular}{lcc}
\toprule
Method & Time / step (s) $\downarrow$ & Overhead vs. RLHF (\%) $\downarrow$ \\
\midrule
RLHF     & 11.67 & 0.0 \\
ABC      & 13.24 & 13.5 \\
SCAR     & 26.98 & 131.2 \\
S2T-RLHF & 21.18 & 81.5 \\
\bottomrule
\end{tabular}
\end{table}

\paragraph{Component-level runtime breakdown.}
To identify the source of the additional cost, Table~\ref{tab:s2t_runtime_breakdown} decomposes the runtime of S2T-RLHF into the RLHF backbone, perturbation-based sentence influence estimation, sentence-level bargaining, and DTAN refinement. The breakdown shows that the overhead is dominated by sentence influence estimation, which requires 8.81 seconds per step. In contrast, sentence-level bargaining takes only 0.12 seconds per step, and DTAN refinement takes 0.61 seconds per step. Thus, the main computational bottleneck of S2T-RLHF is reward-model evaluation under sentence perturbations, rather than the bargaining solver or the token-level refinement module.

\begin{table}[h]
\centering
\caption{Component-level runtime breakdown of S2T-RLHF. Minor discrepancies are due to rounding.}
\label{tab:s2t_runtime_breakdown}
\begin{tabular}{lccc}
\toprule
Component & Time / step (s) $\downarrow$ & Added over RLHF (s) & Share of total step time (\%) \\
\midrule
Standard RLHF backbone       & 11.67 & --   & 55.1 \\
\midrule
Sentence influence  & 8.81  & 8.81 & 41.6 \\
Sentence bargaining & 0.12  & 0.12 & 0.6 \\
DTAN refinement     & 0.61  & 0.61 & 2.9 \\
\midrule
S2T-RLHF total      & 21.18 & 9.54 & 100.0 \\
\bottomrule
\end{tabular}
\end{table}

\paragraph{Discussion.}
These results indicate that the computational cost of S2T-RLHF is primarily associated with Stage I, especially perturbation-based sentence influence estimation. This cost is expected because Stage I requires additional reward-model evaluations to estimate how sentence-level semantic perturbations affect the sequence-level reward. By contrast, the Nash-style bargaining solver and DTAN-based token refinement are lightweight. Future implementations may further reduce the overhead by caching reward-model evaluations, reducing the number of perturbations, batching sentence perturbations more aggressively, or replacing perturbation-based influence estimation with a learned amortized estimator.